\documentclass[a4paper,fleqn]{cas-dc}

\usepackage[numbers]{natbib}

\usepackage{adjustbox}
\usepackage{bm}
\usepackage{pifont}
\usepackage{algorithm}
\usepackage{algpseudocode}
\usepackage{amsthm}
\usepackage{capt-of}  

\algrenewcommand\algorithmicrequire{\textbf{Input:}}
\algrenewcommand\algorithmicensure{\textbf{Output:}}

\definecolor{lightgray}{gray}{0.9}
\definecolor{changehl}{RGB}{255,245,157}
\definecolor{changetext}{RGB}{190,90,0}

\newcommand{\cmark}{\ding{51}}
\newcommand{\xmark}{\ding{55}}

\newtheorem{assumption}{Assumption}
\newtheorem{lemma}{Lemma}
\newtheorem{theorem}{Theorem}

\begin{document}
\let\WriteBookmarks\relax
\def\floatpagepagefraction{1}
\def\textpagefraction{.001}

\shorttitle{Continuous-Time Probabilistic Correctors for Spacecraft Trajectory Forecasting}
\shortauthors{M. B. Shahid et~al.}

\title[mode=title]{Continuous-Time Probabilistic Correctors for Uncertainty-Aware Physics-Based Spacecraft Trajectory Forecasting}

\author[1]{Muhammad Bilal Shahid}
\ead{belal@iastate.edu}
\credit{Conceptualization, Methodology, Software, Data curation, Writing -- original draft, Visualization}

\author[1]{Zhanhong Jiang}
\ead{zhjiang@iastate.edu}
\credit{Formal analysis, Writing -- review and editing}

\author[1]{Soumik Sarkar}
\ead{soumiks@iastate.edu}
\credit{Supervision, Writing -- review and editing}

\author[1]{Cody Fleming}
\cormark[1]
\ead{flemingc@iastate.edu}
\credit{Conceptualization, Supervision, Funding acquisition, Writing -- review and editing}

\affiliation[1]{organization={Iowa State University},
    city={Ames},
    state={Iowa},
    postcode={50011},
    country={USA}}

\cortext[cor1]{Corresponding author. Tel.\ 515-294-6938, Howe Hall, 537 Bissell Rd, Ames, IA 50011-1096, United States of America.}

\begin{abstract}
Long-horizon spacecraft trajectory forecasting suffers from error accumulation due to the absence of corrective observations in the forecast regime, making reliable uncertainty estimation crucial for safety-critical decision-making such as space domain awareness and conjunction assessment. While high-fidelity physics-based orbit propagators provide accurate deterministic forecasts, they typically lack calibrated uncertainty estimates over long horizons. We introduce a Predictor--Corrector framework in which a physics-based continuous-time \textit{deterministic} forecaster is augmented with a learned continuous-time \textit{probabilistic} Corrector that models forecast errors. The proposed Corrector can be wrapped around an existing deterministic propagator to improve forecast accuracy while producing sharp and calibrated full-covariance uncertainty estimates. The Corrector is based on Latent Neural Controlled Differential Equations (Latent NCDEs) and models the probabilistic temporal evolution of forecast errors in continuous time, naturally supporting irregular sampling and missing features. We further introduce a loss function that promotes calibration and sharpness in long-horizon uncertainty propagation. We evaluate the proposed framework on long-horizon spacecraft trajectory forecasting using real-world data from NASA's Crustal Dynamics Data Information System (CDDIS), wrapping the Corrector around NASA's General Mission Analysis Tool (GMAT). Across forecast horizons of 2--4 days without observations and six rolling test windows, the proposed approach consistently improves accuracy and uncertainty calibration compared to deterministic baselines and Latent ODE-based correctors, demonstrating the effectiveness of the continuous-time probabilistic Corrector for trajectory forecasting.
\end{abstract}

\begin{highlights}
\item A Predictor--Corrector framework wraps GMAT for uncertainty-aware orbit forecasting.
\item Latent Neural CDEs model forecast-error evolution in continuous time.
\item Student-$t$ likelihood with CRPS yields sharp, calibrated uncertainty.
\item Validated on real NASA CDDIS spacecraft data over 2--4 day open-loop horizons.
\item Improves accuracy and calibration over GMAT and Latent ODE baselines.
\end{highlights}


\maketitle

\section{Introduction}\label{sec:intro}

Spacecraft trajectory forecasting underpins core space domain awareness (SDA) functions, including orbit determination, collision-risk assessment, and the planning of collision-avoidance maneuvers \cite{cara_handbook,aida2016conjunction}. These functions demand not only accurate long-horizon state predictions but also trustworthy uncertainty estimates, because downstream conjunction-risk decisions depend directly on the predicted error covariance \cite{kerr2021state,hill2012conjunction}. Although high-fidelity physics-based propagators, most notably NASA's General Mission Analysis Tool (GMAT) \cite{hughes2016general}, generate accurate deterministic trajectories, they do not natively produce calibrated uncertainty over the long, observation-free horizons typical of operational forecasting. Consequently, unmodeled dynamics, parametric mismatch, and numerical approximations accumulate over the forecast window, and the reported error covariance grows increasingly overconfident, understating the true uncertainty and the associated collision risk. This gap has motivated growing interest in machine-learning methods that endow physics-based propagators with reliable, well-calibrated uncertainty for SDA \cite{RR-A2318-1}.

To that end, this work introduces a Predictor--Corrector framework based on a newly proposed Latent Neural Controlled Differential Equation (Latent NCDE) architecture. In this framework, the Predictor is a physics-based GMAT dynamics model, and the Corrector is based on Latent NCDE. Latent NCDE is an encoder--decoder architecture in which both the encoder and the decoder are parametrized using NCDE. The encoder encodes the GMAT Predictor's past error history into a latent distribution, which serves as a prior for the decoder in modeling the probabilistic temporal evolution of GMAT's forecast errors. The goal of the Corrector is to augment the Predictor with a continuous-time, probabilistic representation of its forecast errors over time. The benefits of the proposed Corrector are two-fold. It corrects the GMAT Predictor's deterministic forecasts, improving accuracy while generating reliable uncertainty ellipsoids that capture predictive uncertainty (Fig.~\ref{fig:intro}). Applied to spacecraft trajectory forecasting, the resulting Predictor--Corrector framework produces more accurate and well-calibrated long-horizon forecasts.

\paragraph{Contributions} We:
(i) propose \textbf{Continuous-Time Probabilistic Correctors (CTPC)} that wrap deterministic physics simulators (e.g., GMAT) for uncertainty-aware long-horizon forecasting;
(ii) introduce CTPC-CDE++ based on a newly proposed architecture, \textbf{Latent NCDE}, as a promising CTPC for real-world datasets that supports irregularly sampled observations with missing features;
(iii) design a \textbf{loss function} combining Student-$t$ likelihood and CRPS to produce sharp, full-covariance, calibrated uncertainty over long horizons;
(iv) demonstrate consistent gains in \textbf{accuracy} and \textbf{calibration} of GMAT on \textbf{real-world spacecraft trajectories} across six rolling test windows and horizons up to 4 days, including robust error modeling across Earth-Centered Inertial (ECI) and Radial--Transverse--Normal (RTN) frames.

\par\medskip
\noindent\begin{minipage}{\linewidth}
  \centering
  \includegraphics[width=0.8\linewidth]{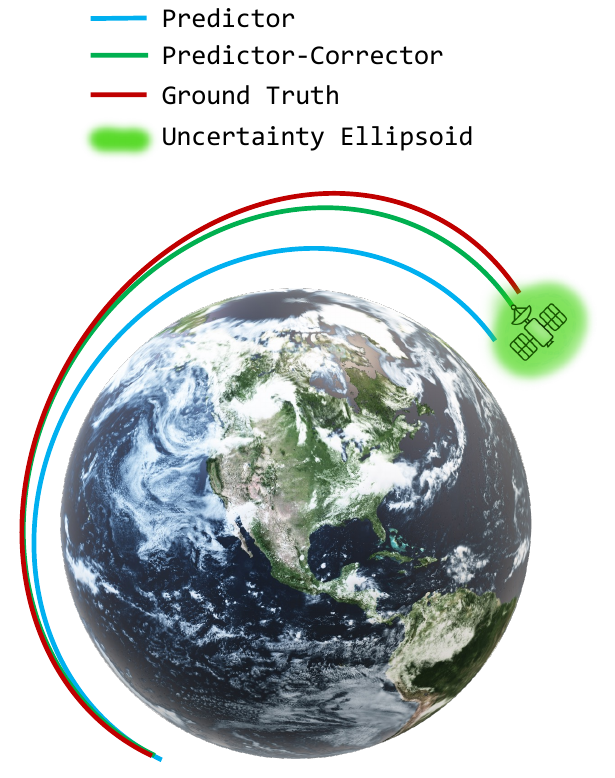}
  \captionof{figure}{Spacecraft trajectory forecasting with continuous-time probabilistic correction.}
  \label{fig:intro}
\end{minipage}
\par\medskip

\section{Related Work}\label{sec:related}

\paragraph{Predictor--Corrector framework}
We briefly review work related to Predictor--Corrector ideas in time-series modeling. In the forecasting literature, several approaches introduce auxiliary models that learn to predict and compensate for forecast errors, although they are not explicitly framed as Predictor--Corrector methods \citep{chen2022machine,slater2023hybrid,liu2025explainable,zhang2022hybrid}. Early examples include hybrid ARIMA--MLP models \citep{zhang2003time} and ARIMA--LSTM approaches \citep{xu2022application}, where a learned component corrects the errors of a statistical forecaster. More recent works formalize this idea through Predictor--Corrector frameworks, where a learned Corrector corrects the deterministic Predictor's forecasts. HopCast is a retrieval-based Corrector but is restricted to regularly sampled data, exhibits weak extrapolation performance, and cannot model full predictive covariance \citep{shahid2025hopcast}. NCDE-based \textit{deterministic} Correctors extend correction to irregularly sampled data, but provide no uncertainty quantification \citep{shahid2025neuralcdescorrectorslearned}. Our work unifies and extends these approaches by introducing continuous-time \textit{probabilistic} Correctors that both correct deterministic forecasts and produce full-covariance uncertainty ellipsoids, while naturally supporting irregular sampling and missing features.

\paragraph{Continuous-time models} Continuous-time sequence models are commonly formulated using Neural Ordinary Differential Equations (NODEs), which model hidden-state dynamics as ODEs and naturally handle irregular sampling \cite{chen2019neuralordinarydifferentialequations}. Variants such as GRU-ODE-Bayes and Latent ODEs combine NODEs with recurrent updates or latent-variable inference to better handle sporadic observations \cite{debrouwer2019gruodebayescontinuousmodelingsporadicallyobserved,rubanova2019latentodesirregularlysampledtime,lechner2020learninglongtermdependenciesirregularlysampled}, but remain limited by uncontrolled dynamics. NCDEs address this by explicitly conditioning hidden-state evolution on an input-driven control path, yielding a continuous-time analogue of recurrent neural networks (RNNs) with improved modeling of long-term dependencies and irregularly sampled data \cite{kidger2020neuralcontrolleddifferentialequations}. Extensions include variants for long sequences and streaming settings \cite{morrill2021neuralroughdifferentialequations,morrill2021neural}. Recently, Structured Linear CDEs (SLiCE) introduced linear CDE formulations inspired by classical and selective state-space models \cite{walker2025structuredlinearcdesmaximally,gu2024mamba,muca2024theoretical}, enabling parallelism and maximal expressivity. However, their linear structure was insufficient for capturing the nonlinear, trajectory-dependent error dynamics in our long-horizon spacecraft trajectory forecasting experiments, motivating our focus on nonlinear NCDE-based probabilistic correctors.

\paragraph{Uncertainty quantification for SDA} Reliable SDA requires uncertainty estimates that accurately characterize the prediction error, as collision-risk screening is governed directly by the predicted state covariance \cite{newman2019requirements}. The fidelity of this covariance is critical: an excessively diffuse covariance can attenuate the computed probability of collision and consequently lead to the underestimation of a genuine close-approach hazard \cite{hejduk2019dilution}. Classical estimators such as the Extended Kalman Filter propagate uncertainty under linearized dynamics and Gaussian assumptions, which frequently degrade over long, observation-free horizons \cite{mashiku2012statistical}. Probabilistic forecasting, by contrast, seeks to characterize the full predictive distribution and is assessed using strictly proper scoring rules such as the Continuous Ranked Probability Score (CRPS) \cite{gneiting2007strictly}, which quantify the discrepancy between the predicted dispersion and the observed error statistics. Recent machine-learning assessments for SDA report that uncertainty-aware models outperform deterministic neural baselines and provide more reliable support for operational decision-making \cite{RR-A2318-2}.

\section{Problem Description}\label{sec:problem}

We consider the problem of uncertainty-aware long-horizon (up to 4 days ahead) spacecraft trajectory forecasting in regimes where \textit{no observations} are available for correction during propagation.
We assume that a high-fidelity physics-based simulator, such as GMAT, propagates an initial observed state of a spacecraft in the ECI frame, $\mathbf{x}_{\text{ECI}}(t_0) \in \mathbb{R}^3$, producing deterministic forecasts of spacecraft states
$\{\mathbf{\hat{x}}_{\text{ECI}}(t_0),\ldots,\mathbf{\hat{x}}_{\text{ECI}}(t_{T^{\prime}}),$ $ \mathbf{\hat{x}}_{\text{ECI}}(t_{T^{\prime}+1}),\ldots,\mathbf{\hat{x}}_{\text{ECI}}(t_{T})\}$
over long horizons $[t_0,\ldots,t_{T^{\prime}},$ $t_{T^{\prime}+1},\ldots,t_T]$, where $\mathbf{\hat{x}}_{\text{ECI}}(t) \in \mathbb{R}^{3}$
\footnote{GMAT internally propagates a full Cartesian state; here, we denote the position component of the spacecraft state by $\hat{\mathbf{x}}_{\text{ECI}}(t)$.}
denotes the spacecraft state's forecast at time $t_i \in \mathbb{R}$ and $t_0 < \ldots < t_{T^{\prime}} < t_{T^{\prime}+1} < \ldots < t_{T}$.
Let $\mathbf{e}_{\text{ECI}}(t) = \mathbf{x}_{\text{ECI}}(t) - \hat{\mathbf{x}}_{\text{ECI}}(t)$ denote the error at time $t$.
We assume that the spacecraft states are observed over the time interval $[t_0,\ldots,t_{T^\prime}]$.
Therefore, the GMAT errors, $\mathbf{e}_{\text{ECI}}(t_{0:T^{\prime}})$ \footnote{We use the shorthand $\mathbf{u}(t_{a:b}) := \{\mathbf{u}(t_i)\}_{i=a}^{b}$ for any time-indexed quantity $\mathbf{u}(t)$.}, over this time interval are known.
The CTPC-CDE++ models the probabilistic evolution of the errors of GMAT over the forecast horizon $[t_{T^{\prime}+1},\ldots,t_T]$.
Specifically, the CTPC-CDE++ predicts errors, $\mathbf{\hat{e}}_{\text{ECI}}(t) \mid \mathbf{e}_{\text{ECI}}(t_{0:T^{\prime}}),\mathbf{\hat{x}}_{\text{ECI}}(t_{T^{\prime}+1:T}) \sim \mathcal{T}_{\nu}(\boldsymbol{\mu}_t,\boldsymbol{\Sigma}_t)$ for each forecast $\mathbf{\hat{x}}_{\text{ECI}}(t)$ from $t \in [t_{T^{\prime}+1},t_{T}]$, where $\mathcal{T}_{\nu}$ denotes the Student-$t$ distribution with $\nu$ degrees of freedom.
For a forecast $\hat{\mathbf{x}}_{\text{ECI}}(t) \in [t_{T^{\prime}+1},\ldots,t_T]$, the predicted mean $\boldsymbol{\mu}_t$ corrects, and the predicted covariance $\boldsymbol{\Sigma}_t$ defines the uncertainty ellipsoids as follows:

\begin{align}
    \hat{\mathbf{x}}_{\text{ECI}}^{\text{corr}}(t) &=
                                    \hat{\mathbf{x}}_{\text{ECI}}(t)
                                    + \boldsymbol{\mu}_t \\
    \mathcal{E}_{\alpha}(t) &=
                        \left\{
                        \mathbf{e}\in\mathbb{R}^3:
                        (\mathbf{e}-\boldsymbol{\mu}_t)^{\top}
                        \boldsymbol{\Sigma}_{t}^{-1}
                        (\mathbf{e}-\boldsymbol{\mu}_t)
                        \le
                        \chi^{2}_{3}(\alpha)
                        \right\}
\end{align}

The uncertainty ellipsoid $\mathcal{E}_{\alpha}(t)$ defines the uncertainty associated with the GMAT forecast $\hat{\mathbf{x}}_{\text{ECI}}(t)$ at an arbitrary confidence level $\alpha \in (0,1)$. The $\chi^{2}_{3}$ denotes the three-dimensional chi-squared distribution and $\hat{\mathbf{x}}^{\text{corr}}_{\text{ECI}}$ denotes the GMAT's corrected forecast.

\section{Methodology}\label{sec:method}

We introduce a Predictor--Corrector framework in which the Predictor is a continuous-time \textit{deterministic forecaster}, and the Corrector provides continuous-time \textit{probabilistic corrections} for uncertainty-aware long-horizon spacecraft trajectory forecasting, shown in Fig.~\ref{fig:prob_desc}. In this framework, a high-fidelity physics-based simulator, such as GMAT, serves as the Predictor and Latent NCDE as a Corrector.

\par\medskip
\noindent\begin{minipage}{\linewidth}
  \centering
  \includegraphics[width=0.9\linewidth]{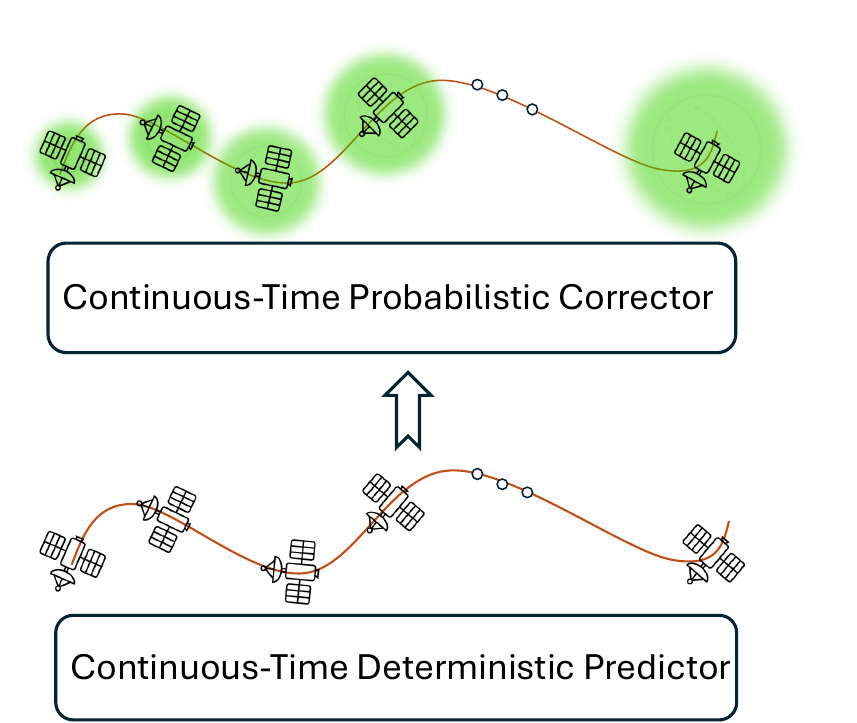}
  \captionof{figure}{The Predictor--Corrector framework. The green glow denotes the uncertainty ellipsoids around the Predictor's forecast.}
  \label{fig:prob_desc}
\end{minipage}
\par\medskip

We propose several variants of CTPC (see Table~\ref{tab:architectures}), and the most generic one, based on Latent NCDE, i.e., CTPC-CDE++, is shown in Fig.~\ref{fig:CTPC_CDE++}. The Latent NCDE is an encoder--decoder architecture wherein the encoder and decoder are based on NCDE. The NCDE is a continuous-time RNN, whose hidden state $\mathbf{z}(t)$ evolves continuously in time driven by the control path $\mathbf{X}(t)$. We explain the Latent NCDE's architecture, shown in Fig.~\ref{fig:CTPC_CDE++}, in the following section.

\begin{figure*}[t]
  \centering
  \includegraphics[width=0.95\textwidth]{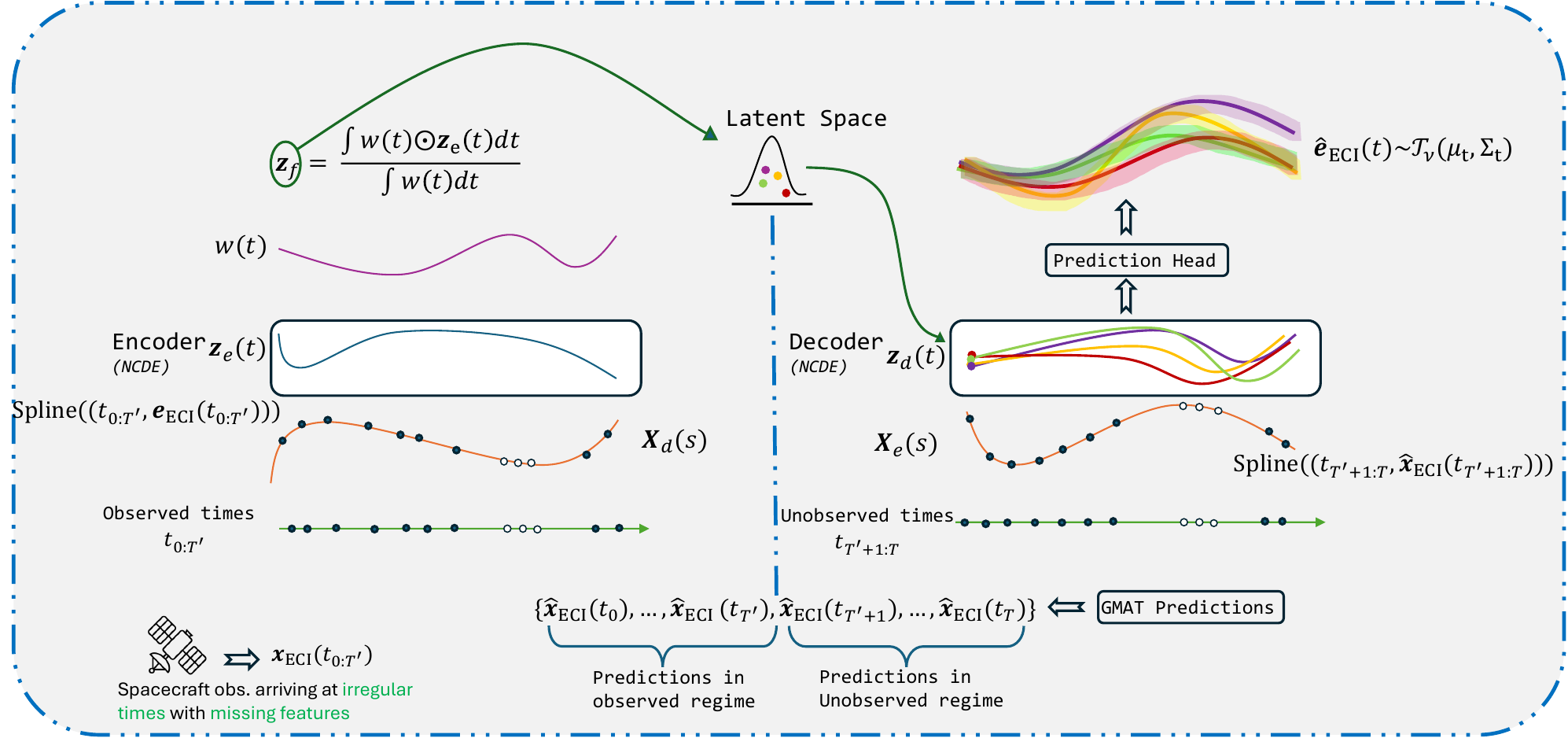}
  \caption{The illustration demonstrates the most general and flexible formulation of our proposed Corrector, i.e., CTPC-CDE++, based on Latent Neural Controlled Differential Equations (Latent NCDE). Though the predicted error $\hat{\mathbf{e}}_{\text{ECI}}(t)$ is shown in the ECI frame, CTPC-CDE++ can model the error equally well in the RTN frame.}
  \label{fig:CTPC_CDE++}
\end{figure*}

\subsection{CTPC-CDE++}\label{sec:CTPC_CDE++}
\paragraph{Encoder (NCDE).} The encoder, consisting of an NCDE, processes the past errors of the GMAT Predictor to learn a latent space distribution. Let $\mathbf{z}_{e}: [t_0,t_{{T}^{\prime}}] \rightarrow \mathbb{R}^{h_e}$, where $t_i \in \mathbb{R}$, denote the hidden states of the encoder. The control path can be written as $\mathbf{X}_{e}(t): [t_0,t_{{T}^{\prime}}] \rightarrow \mathbb{R}^{4}$. Mathematically, for $t \in (t_0,t_{T^{\prime}}]$,

\begin{align}
    \mathbf{z}_{e}(t) &= \mathbf{z}_{e}(t_0) + \int_{t_0}^{t} f_{\theta_{e}}(\mathbf{z}_{e}(s))\,d\mathbf{X}_{e}(s)
    \label{eq:ncde_encoder}
\end{align}

The integral is a Riemann--Stieltjes integral. The vector field $f_{\theta_e}: \mathbb{R}^{h_e} \rightarrow \mathbb{R}^{h_e \times 4}$ is parametrized by a multi-layer perceptron (MLP). The $\mathbf{z}_{e}(t_0) \in \mathbb{R}^{h_e}$ denotes the initial hidden state of the encoder. Based on $\mathbf{z}_{e}(t_0)$ and $\mathbf{X}_{e}(t)$, the encoder learns the evolution of the hidden state $\mathbf{z}_{e} \in (t_0,t_{T^{\prime}}]$. The normalized weighted combination of the $\mathbf{z}_{e}(t)$ gives

\begin{align}
    \mathbf{z}_{f} &= \frac{\int_{t_{0}}^{t_{T^{\prime}}} w(t) \odot \mathbf{z}_{e}(t) dt}{\int_{t_0}^{t_{T^{\prime}}} w(t)dt}
    \label{eq:trapezoidal}
\end{align}

The hidden state $\mathbf{z}_f$ is transformed into the latent space distribution, which is an $L$-dimensional diagonal Gaussian with the following parameters: ${\boldsymbol\mu}_L = (\mu_1,\ldots,\mu_L) \in \mathbb{R}^{L}$ and $\boldsymbol{\Sigma}_L = \mathrm{diag}(\sigma_1^{2},\ldots,\sigma_L^{2}) \in \mathbb{R}^{L \times L}$. The weights $w(t)$ are parametrized via an MLP, $f_{\theta_w}$, called the weight net.

\paragraph{Decoder (NCDE).} The decoder, consisting of an NCDE, processes the forecasts of the GMAT Predictor to predict their time-varying probabilistic corrections. Let $\mathbf{z}_{d}: [t_{T^{\prime}+1},t_{T}] \rightarrow \mathbb{R}^{h_d}$ represent the evolution of the hidden states of the decoder. Mathematically, for $t \in (t_{T^{\prime}+1},t_{T}]$,

\begin{align}
    \mathbf{z}_{d}(t) &= \mathbf{z}_{d}(t_{T^{\prime}+1}) + \int_{t_{T^{\prime}+1}}^{t} f_{\theta_{d}}(\mathbf{z}_{d}(s))\,d\mathbf{X}_{d}(s)
    \label{eq:ncde_decoder}
\end{align}

$\mathbf{X}_{d}(t): [t_{T^{\prime}+1},t_{T}] \rightarrow \mathbb{R}^{4}$ represents the decoder's control path. The vector field $f_{\theta_d}: \mathbb{R}^{h_d} \rightarrow \mathbb{R}^{h_d \times 4}$ is parametrized by an MLP. The initial hidden state $\mathbf{z}_{d}(t_{{T^{\prime}+1}}) \in \mathbb{R}^{h_d}$ is a linear transformation of $s_L$, where $s_L \sim \mathcal{N}(\boldsymbol{\mu}_{L}, \boldsymbol{\Sigma}_{L})$. By utilizing the samples $\mathbf{z}_{d}(t_{T^{\prime}+1})$ as initial hidden states driven by the control path $\mathbf{X}_d(t)$ consisting of deterministic forecasts, the decoder learns to generate diverse evolutions of hidden states $\mathbf{z}_d(t)$ over the forecast horizon $t\in(t_{T^{\prime}+1}, t_T]$, effectively capturing the uncertainty associated with the forecasts.

\paragraph{Prediction head.} At each forecast time $t\in(t_{T^{\prime}+1}, t_T]$, a probabilistic MLP ($f_{\theta_p}$) maps the decoder hidden state $\mathbf{z}_d(t)$ to the parameters of a multivariate Student-$t$ distribution over the error $\hat{\mathbf{e}}_{\text{ECI}}(t)\in\mathbb{R}^3$. The prediction head outputs a mean $\boldsymbol{\mu}_t\in\mathbb{R}^3$ and a lower-triangular matrix $\mathbf{L}_t\in\mathbb{R}^{3\times 3}$ with strictly positive diagonal entries, which parameterizes a full covariance matrix
\begin{equation}
\boldsymbol{\Sigma}_t = \mathbf{L}_t \mathbf{L}_t^{\top}, \qquad \hat{\mathbf{e}}_{\text{ECI}}(t) \mid \mathbf{z}_d(t) \sim \mathcal{T}_{\nu}\!\left(\boldsymbol{\mu}_t, \boldsymbol{\Sigma}_t\right)
\label{eq:student_t_head}
\end{equation}
where $\nu>0$ denotes the degrees of freedom. Let $\mathbf{r}_t = \mathbf{e}_{\text{ECI}}(t) - \boldsymbol{\mu}_t$ and define $\mathbf{y}_t$ via the triangular solve as

\begin{equation}
\mathbf{y}_t = \mathbf{L}_t^{-1}\mathbf{r}_t,
\qquad
\delta_t = \|\mathbf{y}_t\|_2^2
= \mathbf{r}_t^{\top}\boldsymbol{\Sigma}_t^{-1}\mathbf{r}_t,
\end{equation}
where $\delta_t$ is the squared Mahalanobis distance. The log-determinant is computed from the Cholesky factor as
\begin{equation}
\log|\boldsymbol{\Sigma}_t| = 2\sum_{i=1}^{3}\log\big((\mathbf{L}_t)_{ii}\big).
\end{equation}

\subsection{Training}\label{sec:training}

The training procedure is summarized in Algorithm~\ref{alg:train}. The loss consists of three terms, i.e., negative log-likelihood (NLL), continuous ranked probability score (CRPS), and KL regularization.

\paragraph{NLL.} To capture heavy-tailed error distributions \cite{mashiku2012statistical}, we employ the Student-$t$ NLL \cite{murphy2023probabilistic},
\begin{align}
\mathcal{L}_t ^{\text{NLL}}
&= -\log p(\mathbf{r}_t) \nonumber\\
&= \log\Gamma\!\left(\frac{\nu + D}{2}\right)
-\log\Gamma\!\left(\frac{\nu}{2}\right) \nonumber\\
&\quad +\frac{1}{2}\left(D\log(\nu\pi) + \log|\boldsymbol{\Sigma}_t|\right) \nonumber\\
&\quad +\frac{\nu + D}{2}\log\left(1 + \frac{\delta_t}{\nu}\right) \\
\colorbox{green!12}{$\mathcal{L} ^{\text{NLL}}$} &= \frac{1}{T-T^{\prime}-1}\sum_{t=T^{\prime}+1}^{T}(\mathcal{L}_t ^{\text{NLL}})
\label{eq:studentt_nll_chol}
\end{align}

Here, $\Gamma(\cdot)$ denotes the Gamma function and $D = 3$. We learn $\nu$ via a constrained parameter $\nu_{\mathrm{raw}}$ using $\nu = \nu_{\min} + \mathrm{softplus}(\nu_{\mathrm{raw}})$ with $\nu_{\min}=4.5>0$ to avoid degenerate heavy tails \cite{dugas2001incorporating}.

\paragraph{CRPS.} To encourage well-calibrated predictive distributions \cite{gneiting2007strictly}, we also minimize the CRPS as a reconstruction loss. Let $\{\hat{\mathbf{e}}_{\text{ECI}}^{k}(t_{T^{\prime}+1:T})\}_{k=1}^{K}$ denote Monte Carlo samples of predicted forecast errors generated using $K$ samples ($s_L$) from the latent space $\mathcal{N}(\boldsymbol{\mu}_L, \boldsymbol{\Sigma}_L)$, where $\hat{\mathbf{e}}_{\text{ECI}}^{k}(t_i) \in \mathbb{R}^{3}$.
The CRPS loss is defined as
\begin{equation}
\colorbox{blue!12}{$\mathcal{L}^{\mathrm{CRPS}}$}
=
\frac{1}{K}
\sum_{k=1}^{K}
\sum_{j=T^{\prime}+1}^{T}
\sum_{i=1}^{D=3}\!\left(\,\big|\hat{e}_{i,j}^{k} - e_{i,j}\big|
-
\frac{1}{2} \big|\hat{e}_{i,j}^{k} - \hat{e}_{i,j}^{k'}\big|
\right),
\end{equation}
where $\hat{e}_{i,j}^{k}$ and $\hat{e}_{i,j}^{k'}$ are any two independent samples from the latent space $\mathcal{N}(\boldsymbol{\mu}_L, \boldsymbol{\Sigma}_L)$. The $e_{i,j}$ denotes the ground-truth error of the $i^\text{th}$ dimension at the $j^\text{th}$ time.

\paragraph{KL regularization \cite{kingma2013auto}.} The encoder defines a diagonal Gaussian latent posterior $q(\mathbf{z}_0) = \mathcal{N}\!\left(\boldsymbol{\mu}_L,\boldsymbol{\Sigma}_L\right)$
with prior $p(\mathbf{z}_0)=\mathcal{N}(\mathbf{0},\mathbf{I})$.
The KL divergence is
\begin{equation}
\colorbox{red!12}{$\mathcal{L}^{\mathrm{KL}}$} =\frac{1}{2}\sum_{i=1}^{L}\left(\mu_{L,i}^2+\sigma_{L,i}^2-\log \sigma_{L,i}-1\right),
\end{equation}
where $\mu_{L,i}$ and $\sigma_{L,i}^2$ denote the mean and variance of the $i$-th latent dimension, $i=1,\dots,L$, respectively. The final training objective is
\begin{equation}
\mathcal{L}_{\mathrm{total}}
= \colorbox{green!12}{$\mathcal{L}^{\mathrm{NLL}}_t$}
+ \colorbox{blue!12}{$\mathcal{L}^{\mathrm{CRPS}}$}+
\colorbox{red!12}{$\mathcal{L}^{\mathrm{KL}}$}
\end{equation}

\subsection{Inference}\label{para:moment_aggr}
The inference procedure is summarized in Algorithm~\ref{alg:infer}. The moment aggregation details are deferred to Appendix~\ref{adx:app_inf}.

\paragraph{Coordinate-frame convention.} For the main paper, we report results for CTPCs trained and evaluated on the errors in the RTN frame (Table~\ref{tab:main_results}), while the observations and GMAT's forecasts \textit{always} remain in the ECI frame. This is standard for orbit-relative error analysis \cite{kerr2021state}. The conversion of errors between ECI and RTN frames is described in Appendix~\ref{adx:frame_conversion}.
To assess robustness to coordinate-frame choice, we additionally report results for CTPCs trained and evaluated on the errors in the ECI frame in Appendix~\ref{adx:eci_frame} (Table~\ref{tab:main_results_eci}). These results show that the NCDE-based CTPC variants (e.g., CTPC-CDE++) generalize reliably across coordinate frames (see Section~\ref{adx:corr_archi}).

\section{Theoretical Analysis}\label{sec:theory_analysis}
Define the forecast error $\mathbf{e}_\text{ECI}(t)=\mathbf{x}_\text{ECI}(t)-\hat{\mathbf{x}}_\text{ECI}(t)$. We assume access to observations over a warm-up interval $[t_0,t_{T'}]$, and open-loop forecasting over $(t_{T'},t_T]$. The Probabilistic Corrector models the conditional distribution of forecast errors in the forecast regime:
$
    \hat{\mathbf{e}}_\text{ECI}(t)\mid\mathcal{F}_{T'}\sim\mathcal{T}_\nu(\bm{\mu}_t,\Sigma_t), \;t\in(t_{T'},t_T]
$
where $\mathcal{F}_{T'}$ is the sigma-algebra generated by past errors and predictor forecasts, and $\bm{\mu}_t\in\mathbb R^D$ and $\Sigma_t\in\mathbb R^{D\times D}$ (with D=3 the spacecraft-position error dimension) are outputs of a Latent NCDE decoder. Therefore, the corrected forecast is
$
    \hat{\mathbf{x}}^{\text{corr}}_\text{ECI}(t)=\hat{\mathbf{x}}_\text{ECI}(t)+\bm{\mu}_t.
$

\begin{algorithm}[t]
\caption{Training with CTPC-CDE++}
\label{alg:train}
\begin{algorithmic}[1]
\Require 
$N$ trajectories of observed spacecraft states $\mathbf{x}_{\text{ECI}}(t_{0:T})$ and GMAT deterministic forecasts $\hat{\mathbf{x}}_{\text{ECI}}(t_{0:T})$
\Ensure 
Trained Corrector parameters $(\theta_e,\theta_w,\theta_d,\theta_p)$

\State Initialize encoder $f_{\theta_e}$, decoder $f_{\theta_d}$, weight net $f_{\theta_w}$, and prediction head $f_{\theta_p}$
\For{epoch $=1,\dots,E$}
    \State $\mathbf{e}_{\text{ECI}}(t_{0:T}) \leftarrow \mathbf{x}_{\text{ECI}}(t_{0:T}) - \hat{\mathbf{x}}_{\text{ECI}}(t_{0:T})$
    \State $\mathbf{X}_e(s) \leftarrow \mathtt{CubicSpline}\!\big((t_{0:T^\prime}, \mathbf{e}_{\text{ECI}}(t_{0:T^{\prime}}))\big)$
    \State $\mathbf{z}_e(t_0) \leftarrow \mathbf{0}_{\mathbb{R}^{h_e}}$
    \State $\mathbf{z}_e(t_{0:T'}) \leftarrow f_{\theta_e}\!\left(\mathbf{z}_e(t_0), \mathbf{X}_e(s)\right)$
    \State $w(t_{0:T'}) \leftarrow f_{\theta_w}\!\big(\mathbf{z}_e(t_{0:T'})\big)$
    \State $\mathbf{z}_f \leftarrow \mathtt{Trapezoidal}\!\left((w(t_{0:T'}), \mathbf{z}_e(t_{0:T'}))\right)$ \quad // Eq.~\ref{eq:trapezoidal}
    \State $(\boldsymbol{\mu}_L,\boldsymbol{\Sigma}_L) \leftarrow \mathbf{W}_{h_e \rightarrow L}\mathbf{z}_f$
    \For{$k=1,\dots,K$}
        \State Sample $s_L^{(k)} \sim \mathcal{N}(\boldsymbol{\mu}_L,\boldsymbol{\Sigma}_L)$
        \State $\mathbf{z}_d^{(k)}(t_{T'+1}) \leftarrow \mathbf{W}_{L \rightarrow h_d}s_L^{(k)}$
        \State $\mathbf{X}_d(s) \leftarrow \mathtt{CubicSpline}\!\big((t_{T'+1:T}, \hat{\mathbf{x}}_{\text{ECI}}(t_{T'+1:T}))\big)$
        \State $\mathbf{z}_d^{(k)}(t) \leftarrow f_{\theta_d}\!\left(\mathbf{z}_d^{(k)}(t_{T'+1}), \mathbf{X}_d(s)\right)$
        \State $(\boldsymbol{\mu}_t^{(k)},\mathbf{L}_t^{(k)}) \leftarrow f_{\theta_p}\!\big(\mathbf{z}_d^{(k)}(t)\big)$
    \EndFor
    \State $\mathcal{L}_{\mathrm{total}} \leftarrow \mathcal{L}^{\mathrm{NLL}} + \mathcal{L}^{\mathrm{CRPS}} + \mathcal{L}^{\mathrm{KL}}$
    \State Update $(\theta_e,\theta_w,\theta_d,\theta_p)$ via gradient descent
\EndFor
\State \textbf{return} $(\theta_e,\theta_w,\theta_d,\theta_p)$
\end{algorithmic}
\end{algorithm}

\begin{algorithm}[t]
\caption{Inference with CTPC-CDE++}
\label{alg:infer}
\begin{algorithmic}[1]
\Require 
Observed spacecraft states for warm-up $\mathbf{x}_{\text{ECI}}(t_{0:T'})$, GMAT deterministic forecasts $\hat{\mathbf{x}}_{\text{ECI}}(t_{0:T})$
\Ensure 
Corrected forecasts $\hat{\mathbf{x}}_{\text{ECI}}^{\text{corr}}(t_{T'+1:T})$ and uncertainty ellipsoids $\mathcal{E}_{\alpha}(t_{T'+1:T})$

\State $\mathbf{e}_{\text{ECI}}(t_{0:T'}) \leftarrow \mathbf{x}_{\text{ECI}}(t_{0:T'}) - \hat{\mathbf{x}}_{\text{ECI}}(t_{0:T'})$
\State $\mathbf{X}_e(s) \leftarrow \mathtt{CubicSpline}\!\big((t_{0:T'}, \mathbf{e}_{\text{ECI}}(t_{0:T'}))\big)$
\State $\mathbf{z}_e(t_0) \leftarrow \mathbf{0}_{\mathbb{R}^{h_e}}$
\State $\mathbf{z}_e(t_{0:T'}) \leftarrow f_{\theta_e}\!\left(\mathbf{z}_e(t_0), \mathbf{X}_e(s)\right)$
\State $w(t_{0:T'}) \leftarrow f_{\theta_w}\!\big(\mathbf{z}_e(t_{0:T'})\big)$
\State $\mathbf{z}_f \leftarrow \mathtt{Trapezoidal}\!\left((w(t_{0:T'}), \mathbf{z}_e(t_{0:T'}))\right)$
\State $(\boldsymbol{\mu}_L,\boldsymbol{\Sigma}_L) \leftarrow \mathbf{W}_{h_e \rightarrow L}\mathbf{z}_f$

\For{$k=1,\dots,K$}
    \State Sample $s_L^{(k)} \sim \mathcal{N}(\boldsymbol{\mu}_L,\boldsymbol{\Sigma}_L)$
    \State $\mathbf{z}_d^{(k)}(t_{T'+1}) \leftarrow \mathbf{W}_{L\rightarrow h_d}s_L^{(k)}$
    \State $\mathbf{X}_d(s) \leftarrow \mathtt{CubicSpline}\!\big((t_{T'+1:T}, \hat{\mathbf{x}}_{\text{ECI}}(t_{T'+1:T}))\big)$
    \State $\mathbf{z}_d^{(k)}(t) \leftarrow f_{\theta_d}\!\left(\mathbf{z}_d^{(k)}(t_{T'+1}), \mathbf{X}_d(s)\right)$
    \State $(\boldsymbol{\mu}_{t}^{(k)},\mathbf{L}_{t}^{(k)}) \leftarrow f_{\theta_p}\!\big(\mathbf{z}_d^{(k)}(t)\big)$
\EndFor

\State $\{\bar{\boldsymbol{\mu}}_t,\bar{\boldsymbol{\Sigma}}_t\}_{t=T'+1}^{T}$ from $\{(\boldsymbol{\mu}_t^{(k)},\mathbf{L}_t^{(k)})\}_{k=1}^{K}$ \quad // Appendix~\ref{adx:app_inf}

\State $\hat{\mathbf{x}}_{\text{ECI}}^{\text{corr}}(t) \leftarrow \hat{\mathbf{x}}_{\text{ECI}}(t) + \bar{\boldsymbol{\mu}}_t$ \& $\mathcal{E}_{\alpha}(t)$ from $(\bar{\boldsymbol{\mu}}_t,\bar{\boldsymbol{\Sigma}}_t)$ 

\State \textbf{return} $\hat{\mathbf{x}}^{\text{corr}}_{\text{ECI}}(t_{T'+1:T})$, $\mathcal{E}_{\alpha}(t_{T'+1:T})$
\end{algorithmic}
\end{algorithm}

As shown in Eq.~\ref{eq:ncde_encoder}, the final encoder state $\mathbf{z}_e(t_{T'})$ parameterizes a latent distribution for the decoder initialization $\mathbf{z}_d(t_{T'+1})$.
Likewise, the decoder evolves latent states forward over $(t_{T'},t_T]$ in Eq.~\ref{eq:ncde_decoder}.
The outputs are $\bm{\mu}_t$ and $\Sigma_t$ as shown in Eq.~\ref{eq:student_t_head}.
Given these notations and setup, we are now ready to state assumptions to characterize the analysis. $\|\cdot\|$ is the vector norm or matrix norm.
\begin{assumption}\label{assumption_4}
    The decoder vector field $f_{\theta_d}$ satisfies
    \[
    \|f_{\theta_d}(\mathbf{z}_d)-f_{\theta_d}(\mathbf{z}'_d)\|\leq L_d\|\mathbf{z}_d-\mathbf{z}'_d\|,
    \]
    for all $\mathbf{z}_d, \mathbf{z}'_d$, with $L_d>0$.
\end{assumption}
This assumption ensures well-posedness of controlled differential equations and can be enforced via spectral normalization.
\begin{assumption}\label{assumption_5}
    The control paths $\mathbf{X}_e$ and $\mathbf{X}_d$ driving the NCDE encoder and decoder have bounded total variation: $\|\mathbf{X}_e\|_{TV}<\infty$ and $\|\mathbf{X}_d\|_{TV}<\infty$.
\end{assumption}
As the predictor trajectories are spline-interpolated, they satisfy this condition automatically. Next, we also impose another assumption on the prediction head.
\begin{assumption}\label{assumption_6}
    There exist constants $L_\mu,L_\Sigma>0$ such that $\|\bm{\mu}(\mathbf{z}_d)-\bm{\mu}(\mathbf{z}'_d)\|\leq L_\mu\|\mathbf{z}_d-\mathbf{z}'_d\|$ and $\|\Sigma(\mathbf{z}_d)-\Sigma(\mathbf{z}'_d)\|\leq L_\Sigma\|\mathbf{z}_d-\mathbf{z}'_d\|$.
\end{assumption}
The above assumption plays the role of a regularity condition, particularly in a standard neural network, which can be enforced by spectral normalization. The final assumption to be made is the bounded latent distribution for the decoder hidden state initialization.
\begin{assumption}\label{assumption_7}
    The latent distribution induced by the encoder has a bounded second moment such that $\|\mathbf{z}_d(t_{T'+1})\|^2\leq M_0<\infty$.
\end{assumption}
Under these assumptions, we establish theoretical guarantees for CTPC. We first analyze decoder stability and long-horizon covariance growth, showing that the open-loop behavior is governed entirely by the decoder.
\begin{lemma}\label{lemma_3}
    Let all assumptions hold. The decoder states satisfy the following relationship for all $t\in(t_{T'},t_T]$,
    \begin{equation}
        \mathbb E\|\mathbf{z}_d(t)\|^2\leq M_0\exp(2L_d\|\mathbf{X}_d\|_{TV;[t_{T'},t]}).
    \end{equation}
\end{lemma}
\begin{theorem}\label{theorem_1}
    Let $\mathbf{r}_t:=\mathbf{e}_\text{ECI}(t)-\bm{\mu}_t$.
    Under all assumptions, minimizing the population loss yields
    $
        \mathbb E[\mathbf{r}_t^\top\Sigma_t^{-1}\mathbf{r}_t]=D,\;\forall t\in(t_{T'},t_T].
    $
\end{theorem}
Lemma~\ref{lemma_3} shows that under Lipschitz-continuous decoder dynamics and bounded-variation control paths, latent trajectories remain well-posed and grow at most exponentially with the forecast horizon. Since the predictive mean and covariance are smooth functions of the latent state, this ensures controlled propagation of uncertainty and prevents numerical divergence in long-horizon prediction.
Building on this stability, Theorem~\ref{theorem_1} establishes asymptotic calibration consistency. In the population limit, the predictive distribution produced by the Corrector matches the true conditional error distribution, and the normalized Mahalanobis distance converges to one. This follows from the use of strictly proper scoring rules, which ensure that the true conditional mean and covariance uniquely minimize the training objective, yielding uniformly calibrated uncertainty over the forecast horizon.

\begin{theorem}\label{theorem_2}
    Let $\tilde{\mathbf{e}}_\text{ECI}(t)=\mathbf{x}_\text{ECI}(t)-\hat{\mathbf{x}}^\text{corr}_\text{ECI}(t)$. Then,
    \[
        \mathbb E\|\tilde{\mathbf{e}}_\text{ECI}(t)\|^2=\mathbb E\|\mathbf{e}_\text{ECI}(t)\|^2 - \mathbb E\|\bm{\mu}_t\|^2.
    \]
\end{theorem}

Theorem~\ref{theorem_2} implies that the Corrector strictly reduces expected forecast error unless $\bm{\mu}_t=0$, which theoretically justifies our proposed framework. We next present a result to show the long-horizon covariance growth of our proposed framework.
\begin{theorem}\label{theorem_3}
    Let all assumptions hold.
    There exist explicit constants $K_1:=L_\Sigma^2M_0$ and $K_2:=L^2_\Sigma\sigma^2_{\mathbf{z}_d}$, where $\sigma^2_{\mathbf{z}_d}>0$, such that
    \begin{equation}
        \|\Sigma_t\|\leq K_1e^{2L_d\|\mathbf{X}_d\|_{TV;[t_{T'},t]}}+K_2\int_{t_{T'}}^te^{2L_d\|\mathbf{X}_d\|_{TV;[s,t]}}ds.
    \end{equation}
\end{theorem}
Long-horizon forecasting requires uncertainty to grow in a controlled and interpretable manner. Our analysis shows that the predicted covariance $\Sigma_t$ grows at most exponentially with the horizon, with a rate determined by the decoder's Lipschitz constant and the total variation of the predictor trajectory. This bound rules out both covariance collapse and uncontrolled explosion. Moreover, uncertainty growth is driven by the forecast trajectory's geometry rather than stochastic amplification, yielding smooth, well-calibrated uncertainty over long horizons. Because the bound depends on decoder properties rather than specific data instances, it provides a model-level guarantee on long-horizon uncertainty behavior. We also formalize system-level input--output stability, which characterizes robustness to past observation perturbations and implies calibration consistency across trajectories. The system-level results are deferred to Appendix~\ref{adx:missing_theory}.

\section{A Framework for Continuous-Time Probabilistic Forecasting}\label{sec:framework}

\begin{table*}[!b] 
\centering
\small 
\caption{The CTPC variants based on the type of encoder, decoder, prediction head, and loss function. Additionally, these variants are classified by the functionalities they possess. \cmark = “supported / implemented”; \xmark = “not supported / not implemented”}
\vspace{0.15cm}
\renewcommand{\arraystretch}{1.2}
\setlength{\tabcolsep}{4pt} 
\begin{adjustbox}{max width=1.0\textwidth} 
\begin{tabular}{c|c|c|c|cccc|c|c|c|c|c}
\toprule
\multicolumn{1}{c|}{{\multirow{3}{*}{{\parbox[c][1cm][c]{2.5cm}{\centering \textbf{CTPC}\\ \textbf{Variants}}}}}} & 
\multicolumn{1}{c|}{\multirow{3}{*}{\textbf{Encoder}}} &
\multicolumn{1}{c|}{\multirow{3}{*}{\textbf{Decoder}}} &
\multicolumn{1}{c|}{{\multirow{3}{*}{{\parbox[c][1cm][c]{2.5cm}{\centering \textbf{Prediction}\\ \textbf{Head}}}}}} &
\multicolumn{4}{c|}{\multirow{3}{*}{\textbf{Loss function}}} &
\multicolumn{1}{c|}{\multirow{3}{*}{\textbf{Accuracy}}} &
\multicolumn{1}{c|}{\multirow{3}{*}{\textbf{Calibration}}} &
\multicolumn{1}{c|}{\multirow{3}{*}{\textbf{Irregular Sampling}}} &
\multicolumn{1}{c|}{\multirow{3}{*}{\textbf{Missing Features}}} &
\multicolumn{1}{c}{\multirow{3}{*}{{\parbox[c][1cm][c]{2.5cm}{\centering \textbf{Robust to}\\ \textbf{Coordinate } \\ \textbf{Transformation}}}}}\\
& & & & & & & & & & &\\
& & & & & & & & & & &\\
\midrule
\multicolumn{1}{c|}{\multirow{1}{*}{$\text{Latent ODE}^\star$}} & 
\multicolumn{1}{c|}{\multirow{1}{*}{ODE-RNN}} &
\multicolumn{1}{c|}{\multirow{1}{*}{NODE}} &
\multicolumn{1}{c|}{\multirow{1}{*}{Deterministic MLP}} &
\multicolumn{4}{c|}{\multirow{1}{*}{$\mathcal{L}^{\text{MSE}} + \mathcal{L}^{\text{KL}}$}} &
\multicolumn{1}{c|}{\multirow{1}{*}{\cmark}} &
\multicolumn{1}{c|}{\multirow{1}{*}{\xmark}} &
\multicolumn{1}{c|}{\multirow{1}{*}{\cmark}} &
\multicolumn{1}{c|}{\multirow{1}{*}{\xmark}} &
\multicolumn{1}{c}{\multirow{1}{*}{\xmark}}\\
\midrule
\multicolumn{1}{c|}{\multirow{1}{*}{$\text{Latent ODE}^\dagger$}} & 
\multicolumn{1}{c|}{\multirow{1}{*}{ODE-RNN}} &
\multicolumn{1}{c|}{\multirow{1}{*}{NODE}} &
\multicolumn{1}{c|}{\multirow{1}{*}{Deterministic MLP}} &
\multicolumn{4}{c|}{\multirow{1}{*}{$\mathcal{L}^{\text{CRPS}} + \mathcal{L}^{\text{KL}}$}} &
\multicolumn{1}{c|}{\multirow{1}{*}{\cmark}} &
\multicolumn{1}{c|}{\multirow{1}{*}{\xmark}} &
\multicolumn{1}{c|}{\multirow{1}{*}{\cmark}} &
\multicolumn{1}{c|}{\multirow{1}{*}{\xmark}} &
\multicolumn{1}{c}{\multirow{1}{*}{\xmark}}\\
\midrule
\multicolumn{1}{c|}{\multirow{1}{*}{CTPC-ODE (Fig.~\ref{fig:CTPC_ODE})}} & 
\multicolumn{1}{c|}{\multirow{1}{*}{ODE-RNN}} &
\multicolumn{1}{c|}{\multirow{1}{*}{NODE}} &
\multicolumn{1}{c|}{\multirow{1}{*}{Probabilistic MLP}} &
\multicolumn{4}{c|}{\multirow{1}{*}{$\mathcal{L}^{\text{CRPS}} + \mathcal{L}^{\text{KL}} + \mathcal{L}^{\text{NLL}}$}} &
\multicolumn{1}{c|}{\multirow{1}{*}{\cmark}} &
\multicolumn{1}{c|}{\multirow{1}{*}{\cmark}} &
\multicolumn{1}{c|}{\multirow{1}{*}{\cmark}} &
\multicolumn{1}{c|}{\multirow{1}{*}{\xmark}} & 
\multicolumn{1}{c}{\multirow{1}{*}{\xmark}}\\
\midrule
\multicolumn{1}{c|}{\multirow{1}{*}{CTPC-CDE (Fig.~\ref{fig:CTPC_CDE})}} & 
\multicolumn{1}{c|}{\multirow{1}{*}{ODE-RNN}} &
\multicolumn{1}{c|}{\multirow{1}{*}{NCDE}} &
\multicolumn{1}{c|}{\multirow{1}{*}{Probabilistic MLP}} &
\multicolumn{4}{c|}{\multirow{1}{*}{$\mathcal{L}^{\text{CRPS}} + \mathcal{L}^{\text{KL}} + \mathcal{L}^{\text{NLL}}$}} &
\multicolumn{1}{c|}{\multirow{1}{*}{\cmark}} &
\multicolumn{1}{c|}{\multirow{1}{*}{\cmark}} &
\multicolumn{1}{c|}{\multirow{1}{*}{\cmark}} &
\multicolumn{1}{c|}{\multirow{1}{*}{\xmark}} & 
\multicolumn{1}{c}{\multirow{1}{*}{\cmark}}\\
\midrule
\multicolumn{1}{c|}{\multirow{1}{*}{CTPC-CDE++ (Fig.~\ref{fig:CTPC_CDE++})}} & 
\multicolumn{1}{c|}{\multirow{1}{*}{NCDE}} &
\multicolumn{1}{c|}{\multirow{1}{*}{NCDE}} &
\multicolumn{1}{c|}{\multirow{1}{*}{Probabilistic MLP}} &
\multicolumn{4}{c|}{\multirow{1}{*}{$\mathcal{L}^{\text{CRPS}} + \mathcal{L}^{\text{KL}} + \mathcal{L}^{\text{NLL}}$}} &
\multicolumn{1}{c|}{\multirow{1}{*}{\cmark}} &
\multicolumn{1}{c|}{\multirow{1}{*}{\cmark}} &
\multicolumn{1}{c|}{\multirow{1}{*}{\cmark}} &
\multicolumn{1}{c|}{\multirow{1}{*}{\cmark}} & 
\multicolumn{1}{c}{\multirow{1}{*}{\cmark}}\\
\bottomrule

\end{tabular}
\end{adjustbox}
\label{tab:architectures}
\end{table*}

This section discusses the CTPC variants tested on the spacecraft trajectory forecasting problem. We start with Latent ODE and gradually progress to our most practical variant based on Latent NCDE.

Table~\ref{tab:architectures} summarizes the CTPC variants considered in this work and positions them relative to existing latent continuous-time ODE models. The baseline ($\text{Latent ODE}^{\star}$ \cite{rubanova2019latentodesirregularlysampledtime}) follows the standard ODE-RNN encoder and NODE decoder with a deterministic prediction head optimized via MSE as the reconstruction loss and KL regularization, yielding accurate point forecasts but lacking calibrated uncertainty estimates. We then replace the reconstruction loss with the CRPS, while keeping the latent KL regularization unchanged ($\text{Latent ODE}^{\dagger}$). This modification isolates the effect of CRPS on calibration, allowing us to directly assess its impact relative to the standard $\text{Latent ODE}$ training objective. Later, we introduce CTPC-ODE, our first variant of CTPC, which augments the Latent ODE architecture with a probabilistic prediction head and likelihood-based training, enabling calibrated uncertainty estimation while retaining support for irregular sampling. Replacing the NODE decoder with an NCDE yields CTPC-CDE, our second variant of CTPC, which further improves robustness to coordinate transformations by modeling forecast corrections as a continuous-time function of a control path consisting of GMAT's deterministic forecasts (Appendix~\ref{adx:eci_frame}).

While ODE-RNN encoders are naturally suited to irregular timestamps via continuous-time hidden dynamics \textit{between} observations, handling partially missing features requires explicit missingness indicators (e.g., mask/time-gap channels) and masked GRU updates. NCDEs admit the same strategy more directly by augmenting the spline-based control path with missingness channels, enabling continuous-time modeling without discarding partially observed measurements. To that end, we introduce CTPC-CDE++, our third variant of CTPC, which employs NCDEs in both the encoder and decoder, providing native support for irregularly sampled observations with missing features via control paths, while maintaining accurate, calibrated, and coordinate-robust probabilistic forecasts.

This progression highlights how each architectural choice contributes specific capabilities, culminating in a flexible CTPC variant, i.e., CTPC-CDE++, for long-horizon continuous-time probabilistic forecasting.

\section{Evaluation Metrics}\label{sec:metrics}

We evaluate the proposed Predictor--Corrector framework using metrics that jointly assess
\textit{forecast accuracy}, \textit{uncertainty calibration}, and \textit{uncertainty sharpness}
for long-horizon, open-loop trajectory forecasting. Since no observations are available in the
forecast regime, all metrics are computed without corrective feedback.

\paragraph{Accuracy (MSE).}
Forecast accuracy is quantified using the mean squared error (MSE) between the true spacecraft
state $\mathbf{x}_{\text{ECI}}(t)$ and the corrected forecast $\hat{\mathbf{x}}_{\text{ECI}}^{\text{corr}}(t)$,
\begin{equation}
\mathrm{MSE}
=
\frac{1}{T}
\sum_{t=T^{\prime}+1}^{T}
\left\|
\mathbf{x}_{\text{ECI}}(t) - \hat{\mathbf{x}}_{\text{ECI}}^{\text{corr}}(t)
\right\|_2^2.
\end{equation}

\paragraph{Calibration (Normalized Mahalanobis Distance).}
Calibration is assessed using the squared Mahalanobis distance between the true forecast error
$\mathbf{e}_{\text{ECI}}(t)$ and the predicted error distribution,
\begin{equation}
d_t^2
=
(\mathbf{e}_{\text{ECI}}(t) - \boldsymbol{\mu}_t)^\top
\boldsymbol{\Sigma}_t^{-1}
(\mathbf{e}_{\text{ECI}}(t) - \boldsymbol{\mu}_t),
\end{equation}
where $\boldsymbol{\mu}_t$ and $\boldsymbol{\Sigma}_t$ denote the predicted mean and covariance
of the forecast error at time $t$.

For a $D$-dimensional multivariate Student-$t$ distribution with degrees of freedom $\nu>2$,
the expected squared Mahalanobis distance satisfies
\begin{equation}
\mathbb{E}[d_t^2] = \frac{\nu}{\nu-2}\,D.
\end{equation}
We therefore report the \emph{normalized Mahalanobis distance}
\begin{equation}
\bar{d}_t^2
=
\frac{d_t^2}{\frac{\nu}{\nu-2}\,D},
\end{equation}
for which a well-calibrated predictive distribution satisfies
\begin{equation}
\mathbb{E}[\bar{d}_t^2] = 1,
\end{equation}
based on Theorem~\ref{theorem_1}.

Averaged over time and trajectories, values close to one
($\bar{d}_t^2 \approx 1$) indicate well-calibrated uncertainty estimates.
Larger values ($\bar{d}_t^2 > 1$) indicate \emph{overconfident} predictions
(i.e., underestimated uncertainty), while smaller values
($\bar{d}_t^2 < 1$) indicate \emph{underconfident} predictions
(i.e., overly diffuse uncertainty).

\paragraph{Sharpness (Log-Determinant).}
Sharpness measures the concentration of the predictive uncertainty independently of the ground
truth. We quantify sharpness using the log-determinant of the predicted covariance matrix, $\log \left| \boldsymbol{\Sigma}_t \right|$.
Lower values correspond to tighter uncertainty ellipsoids. Since sharpness alone can be misleading, this metric is interpreted jointly with calibration to ensure that uncertainty estimates are both sharp and calibrated.

\section{Results}\label{sec:results}

\subsection{Comparing Corrector Architectures}\label{adx:corr_archi}
Table~\ref{tab:architectures} summarizes the evaluated CTPC variants.
We consider two Latent ODE baselines that match the original encoder--decoder architecture of Latent ODE \cite{rubanova2019latentodesirregularlysampledtime} (ODE-RNN encoder + NODE decoder), differing only in the training objective:
(i) $\text{Latent ODE}^{\star}$ uses the default $\mathcal{L}^{\mathrm{MSE}}+\mathcal{L}^{\mathrm{KL}}$ objective, and
(ii) $\text{Latent ODE}^{\dagger}$ replaces the MSE term with CRPS, i.e., $\mathcal{L}^{\mathrm{CRPS}}+\mathcal{L}^{\mathrm{KL}}$, to isolate the effect of a proper scoring rule, that is, CRPS, on calibration. Both architectures use a deterministic multi-layer perceptron as the prediction head.
We then evaluate our proposed CTPC variants, which augment the latent encoder--decoder with a probabilistic prediction head trained with $\mathcal{L}^{\mathrm{NLL}}$ in addition to CRPS and KL.
Finally, we compare decoder choices (NODE vs.\ NCDE) and an end-to-end NCDE encoder--decoder variant (CTPC-CDE++), which additionally supports missing features via spline-based control paths.

\begin{table*}[!ht]
\centering
\small 
\caption{The performance of GMAT Predictor with our proposed Latent Corrector and baseline Latent ODE Corrector. The performance is evaluated based on three metrics, i.e., accuracy (MSE), calibration (squared Mahalanobis distance, $d^2$), and sharpness (log-determinant, $\log \Sigma$). D = days, H = hours, M = minutes. We present results of Latent ODE with three loss functions, each denoted with different symbols in the superscript. The results reported are in RTN frame with total variance ($\bar{\Sigma}$).}
\vspace{0.15cm}
\renewcommand{\arraystretch}{1.2}
\setlength{\tabcolsep}{4pt} 
\begin{adjustbox}{max width=0.9\textwidth} 
\begin{tabular}{c|c|cccccc|cccccc}
\toprule
\multicolumn{1}{c|}{\multirow{3}{*}{\parbox[c][1cm][c]{2.5cm}{\centering \textbf{Test}\\ \textbf{Period}}}} & 
\multicolumn{1}{c|}{\multirow{3}{*}{\textbf{Model}}} &
\multicolumn{6}{c|}{\textbf{Interpolation Horizons}} &
\multicolumn{6}{c}{\textbf{Extrapolation Horizons}} \\
\cline{3-14} 
& & \multicolumn{3}{c|}{1000M (0D 16H 40M)} & \multicolumn{3}{c|}{2000M (1D 9H 20M)} & \multicolumn{3}{c|}{4000M (2D 18H 40M)} & \multicolumn{3}{c}{5760M (4D 0H 0M)} \\
\cline{3-14} 
& & MSE (\%$\downarrow$) & $\bar{d}^2$ & \multicolumn{1}{c|}{$-\log|\bar{\Sigma}|$} & MSE (\%$\downarrow$) & $\bar{d}^2$ & $-\log|\bar{\Sigma}|$ & MSE (\%$\downarrow$) & $\bar{d}^2$ & \multicolumn{1}{c|}{$-\log|\bar{\Sigma}|$} & MSE (\%$\downarrow$) & $\bar{d}^2$ & $-\log|\bar{\Sigma}|$ \\
\midrule
\multirow{6}{*}{\parbox[c][1cm][c]{2.5cm}{\centering 2017-02-15\\ \textbf{to} 2017-02-28}}
& \multicolumn{1}{c|}{GMAT}
& 0.0105 (---) & --- & \multicolumn{1}{c|}{---} 
& 0.0396 (---) & --- & --- 
& 0.1465 (---) & --- & \multicolumn{1}{c|}{---} 
& 0.2963 (---) & --- & --- \\
\cline{2-14}
& \multicolumn{1}{c|}{GMAT+Latent $\text{ODE}^{\star}$}
& 0.0037 (65\%) & 53.34 & \multicolumn{1}{c|}{37.57} 
& 0.0139 (65\%) & 180.21 & 38.35 
& 0.0717 (51\%) & 2271.43 & \multicolumn{1}{c|}{39.52} 
& 0.1745 (41\%) & 48812.39 & 40.07 \\
\cline{2-14}
& \multicolumn{1}{c|}{GMAT+Latent $\text{ODE}^{\dagger}$}
& 0.0032 (69\%) & 5.22 & \multicolumn{1}{c|}{23.10} 
& 0.0154 (61\%) & 33.07 & 22.47 
& 0.0788 (46\%) & 564.40 & \multicolumn{1}{c|}{22.32} 
& 0.1854 (37\%) & 4481.04 & 22.43 \\
\cline{2-14}
& \multicolumn{1}{c|}{GMAT+CTPC-ODE}
& 0.0021 (80\%) & 0.45 & \multicolumn{1}{c|}{17.85} 
& 0.0081 (79\%) & 0.59 & 17.19 
& 0.0414 (71\%) & 0.75 & \multicolumn{1}{c|}{16.40} 
& 0.1174 (60\%) & 1.03 & 16.04 \\
\cline{2-14}
& \multicolumn{1}{c|}{GMAT+CTPC-CDE}
& 0.0033 (69\%) & 0.43 & \multicolumn{1}{c|}{17.91} 
& 0.0097 (75\%) & 0.60 & 17.34 
& 0.0330 (77\%) & 0.84 & \multicolumn{1}{c|}{16.64} 
& 0.0881 (70\%) & 1.21 & 16.33 \\
\cline{2-14}
& \multicolumn{1}{c|}{GMAT+CTPC-CDE++}
& 0.0028 (73\%) & 0.43 & \multicolumn{1}{c|}{17.74} 
& 0.0083 (78\%) & 0.64 & 17.28 
& 0.0263 (82\%) & 0.98 & \multicolumn{1}{c|}{16.70} 
& 0.0693 (76\%) & 1.47 & 16.46 \\
\midrule
\multirow{6}{*}{\parbox[c][1cm][c]{2.5cm}{\centering 2017-04-15\\ \textbf{to} 2017-04-30}}
& \multicolumn{1}{c|}{GMAT}
& 0.0121 (---) & --- & \multicolumn{1}{c|}{---} 
& 0.0393 (---) & --- & --- 
& 0.1437 (---) & --- & \multicolumn{1}{c|}{---} 
& 0.2841 (---) & --- & --- \\
\cline{2-14}
& \multicolumn{1}{c|}{GMAT+Latent $\text{ODE}^{\star}$}
& 0.0051 (58\%) & 40.57 & \multicolumn{1}{c|}{37.02} 
& 0.0164 (58\%) & 39.57 & 36.92 
& 0.0765 (46\%) & 248.74 & \multicolumn{1}{c|}{36.44} 
& 0.1802 (36\%) & 4219.96 & 36.36 \\
\cline{2-14}
& \multicolumn{1}{c|}{GMAT+Latent $\text{ODE}^{\dagger}$}
& 0.0042 (65\%) & 4.14 & \multicolumn{1}{c|}{23.35} 
& 0.0160 (59\%) & 29.79 & 22.60 
& 0.0773 (46\%) & 998.13 & \multicolumn{1}{c|}{22.51}
& 0.1816 (36\%) & 6172.01 & 22.76 \\
\cline{2-14}
& \multicolumn{1}{c|}{GMAT+CTPC-ODE}
& 0.0038 (68\%) & 0.51 & \multicolumn{1}{c|}{18.25} 
& 0.0139 (64\%) & 0.58 & 17.34 
& 0.0562 (60\%) & 0.74 & \multicolumn{1}{c|}{16.37} 
& 0.1312 (53\%) & 1.02 & 15.97 \\
\cline{2-14}
& \multicolumn{1}{c|}{GMAT+CTPC-CDE}
& 0.0050 (58\%) & 0.42 & \multicolumn{1}{c|}{17.32} 
& 0.0155 (60\%) & 0.55 & 16.86 
& 0.0544 (62\%) & 0.87 & \multicolumn{1}{c|}{16.22} 
& 0.1261 (55\%) & 1.27 & 15.94 \\
\cline{2-14}
& \multicolumn{1}{c|}{GMAT+CTPC-CDE++}
& 0.0047 (60\%) & 0.44 & \multicolumn{1}{c|}{17.62} 
& 0.0129 (67\%) & 0.62 & 17.10 
& 0.0480 (66\%) & 0.99 & \multicolumn{1}{c|}{16.45} 
& 0.1154 (59\%) & 1.44 & 16.18 \\
\midrule
\multirow{6}{*}{\parbox[c][1cm][c]{2.5cm}{\centering 2017-06-15\\ \textbf{to} 2017-06-30}}
& \multicolumn{1}{c|}{GMAT}
& 0.0126 (---) & --- & \multicolumn{1}{c|}{---} 
& 0.0468 (---) & --- & --- 
& 0.1647 (---) & --- & \multicolumn{1}{c|}{---} 
& 0.3200 (---) & --- & --- \\
\cline{2-14}
& \multicolumn{1}{c|}{GMAT+Latent $\text{ODE}^{\star}$}
& 0.0035 (72\%) & 35.09 & \multicolumn{1}{c|}{37.47} 
& 0.0147 (68\%) & 954.35 & 37.97 
& 0.0770 (53\%) & 15814.05 & \multicolumn{1}{c|}{38.82} 
& 0.1843 (42\%) & 52871.08 & 39.01 \\
\cline{2-14}
& \multicolumn{1}{c|}{GMAT+Latent $\text{ODE}^{\dagger}$}
& 0.0031 (75\%) & 5.05 & \multicolumn{1}{c|}{22.98} 
& 0.0161 (65\%) & 101.91 & 22.44 
& 0.0837 (49\%) & 1391.71 & \multicolumn{1}{c|}{22.73} 
& 0.1974 (38\%) & 7429.85 & 23.05 \\
\cline{2-14}
& \multicolumn{1}{c|}{GMAT+CTPC-ODE}
& 0.0030 (76\%) & 0.48 & \multicolumn{1}{c|}{18.51} 
& 0.0097 (79\%) & 0.65 & 17.62 
& 0.0336 (79\%) & 0.84 & \multicolumn{1}{c|}{16.62} 
& 0.0859 (73\%) & 1.12 & 16.18 \\
\cline{2-14}
& \multicolumn{1}{c|}{GMAT+CTPC-CDE}
& 0.0034 (72\%) & 0.39 & \multicolumn{1}{c|}{17.71} 
& 0.0114 (75\%) & 0.56 & 17.15 
& 0.0359 (78\%) & 0.99 & \multicolumn{1}{c|}{16.67} 
& 0.0830 (74\%) & 1.32 & 16.46 \\
\cline{2-14}
& \multicolumn{1}{c|}{GMAT+CTPC-CDE++}
& 0.0040 (68\%) & 0.34 & \multicolumn{1}{c|}{17.09} 
& 0.0116 (75\%) & 0.51 & 16.66 
& 0.0395 (76\%) & 0.85 & \multicolumn{1}{c|}{16.07} 
& 0.0933 (70\%) & 1.20 & 15.75 \\
\midrule
\multirow{6}{*}{\parbox[c][1cm][c]{2.5cm}{\centering 2017-08-15\\ \textbf{to} 2017-08-31}}
& \multicolumn{1}{c|}{GMAT}
& 0.0100 (---) & --- & \multicolumn{1}{c|}{---} 
& 0.0324 (---) & --- & --- 
& 0.1075 (---) & --- & \multicolumn{1}{c|}{---} 
& 0.2003 (---) & --- & --- \\
\cline{2-14}
& \multicolumn{1}{c|}{GMAT+Latent $\text{ODE}^{\star}$}
& 0.0049 (51\%) & 155.89 & \multicolumn{1}{c|}{38.01} 
& 0.0135 (58\%) & 85.08 & 36.93 
& 0.0499 (53\%) & 300.31 & \multicolumn{1}{c|}{37.06} 
& 0.1100 (45\%) & 3029.25 & 37.37 \\
\cline{2-14}
& \multicolumn{1}{c|}{GMAT+Latent $\text{ODE}^{\dagger}$}
& 0.0037 (63\%) & 3.01 & \multicolumn{1}{c|}{23.36} 
& 0.0142 (56\%) & 5.85 & 22.57
& 0.0561 (47\%) & 105.81 & \multicolumn{1}{c|}{22.20} 
& 0.1210 (39\%) & 1107.70 & 22.19 \\
\cline{2-14}
& \multicolumn{1}{c|}{GMAT+CTPC-ODE}
& 0.0037 (63\%) & 0.37 & \multicolumn{1}{c|}{17.85} 
& 0.0136 (57\%) & 0.44 & 17.28
& 0.0413 (61\%) & 0.54 & \multicolumn{1}{c|}{16.54} 
& 0.0885 (55\%) & 0.69 & 16.20 \\
\cline{2-14}
& \multicolumn{1}{c|}{GMAT+CTPC-CDE}
& 0.0041 (58\%) & 0.38 & \multicolumn{1}{c|}{17.93} 
& 0.0150 (53\%) & 0.45 & 17.41 
& 0.0430 (60\%) & 0.57 & \multicolumn{1}{c|}{16.72} 
& 0.0883 (55\%) & 0.74 & 16.40 \\
\cline{2-14}
& \multicolumn{1}{c|}{GMAT+CTPC-CDE++}
& 0.0042 (58\%) & 0.41 & \multicolumn{1}{c|}{18.22} 
& 0.0140 (56\%) & 0.54 & 17.77 
& 0.0461 (57\%) & 0.67 & \multicolumn{1}{c|}{17.12} 
& 0.0937 (53\%) & 0.83 & 16.80 \\
\midrule
\multirow{6}{*}{\parbox[c][1cm][c]{2.5cm}{\centering 2017-10-15\\ \textbf{to} 2017-10-30}}
& \multicolumn{1}{c|}{GMAT}
& 0.0088 (---) & --- & \multicolumn{1}{c|}{---} 
& 0.0334 (---) & --- & --- 
& 0.1201 (---) & --- & \multicolumn{1}{c|}{---} 
& 0.2400 (---) & --- & --- \\
\cline{2-14}
& \multicolumn{1}{c|}{GMAT+Latent $\text{ODE}^{\star}$}
& 0.0030 (65\%) & 51.73 & \multicolumn{1}{c|}{37.76} 
& 0.0124 (62\%) & 81.18 & 38.07 
& 0.0621 (48\%) & 94.43 & \multicolumn{1}{c|}{37.80} 
& 0.1483 (38\%) & 169.35 & 37.48 \\
\cline{2-14}
& \multicolumn{1}{c|}{GMAT+Latent $\text{ODE}^{\dagger}$}
& 0.0029 (66\%) & 2.88 & \multicolumn{1}{c|}{23.39} 
& 0.0127 (62\%) & 10.25 & 22.73 
& 0.0640 (46\%) & 33.69 & \multicolumn{1}{c|}{22.50} 
& 0.1525 (36\%) & 84.81 & 22.50 \\
\cline{2-14}
& \multicolumn{1}{c|}{GMAT+CTPC-ODE}
& 0.0025 (71\%) & 0.36 & \multicolumn{1}{c|}{18.15} 
& 0.0094 (71\%) & 0.47 & 17.53 
& 0.0423 (64\%) & 0.70 & \multicolumn{1}{c|}{16.86} 
& 0.1078 (55\%) & 1.02 & 16.55 \\
\cline{2-14}
& \multicolumn{1}{c|}{GMAT+CTPC-CDE}
& 0.0028 (68\%) & 0.37 & \multicolumn{1}{c|}{18.61} 
& 0.0105 (68\%) & 0.52 & 18.01 
& 0.0412 (65\%) & 0.74 & \multicolumn{1}{c|}{17.28} 
& 0.0970 (59\%) & 1.05 & 16.92 \\
\cline{2-14}
& \multicolumn{1}{c|}{GMAT+CTPC-CDE++}
& 0.0027 (69\%) & 0.31 & \multicolumn{1}{c|}{17.55} 
& 0.0089 (73\%) & 0.43 & 17.12 
& 0.0351 (70\%) & 0.62 & \multicolumn{1}{c|}{16.47} 
& 0.0956 (60\%) & 0.88 & 16.16 \\
\midrule
\multirow{6}{*}{\parbox[c][1cm][c]{2.5cm}{\centering 2017-12-15\\ \textbf{to} 2017-12-31}}
& \multicolumn{1}{c|}{GMAT}
& 0.0124 (---) & --- & \multicolumn{1}{c|}{---} 
& 0.0390 (---) & --- & --- 
& 0.1420 (---) & --- & \multicolumn{1}{c|}{---} 
& 0.2819 (---) & --- & --- \\
\cline{2-14}
& \multicolumn{1}{c|}{GMAT+Latent $\text{ODE}^{\star}$}
& 0.0038 (69\%) & 118.81 & \multicolumn{1}{c|}{38.39} 
& 0.0138 (64\%) & 150.28 & 38.18 
& 0.0755 (46\%) & 2877.62 & \multicolumn{1}{c|}{37.92} 
& 0.1811 (35\%) & 10904.01 & 37.87 \\
\cline{2-14}
& \multicolumn{1}{c|}{GMAT+Latent $\text{ODE}^{\dagger}$}
& 0.0042 (66\%) & 4.32 & \multicolumn{1}{c|}{23.98} 
& 0.0150 (61\%) & 10.99 & 23.04 
& 0.0786 (44\%) & 310.34 & \multicolumn{1}{c|}{22.64} 
& 0.1855 (34\%) & 1176.57 & 22.67 \\
\cline{2-14}
& \multicolumn{1}{c|}{GMAT+CTPC-ODE}
& 0.0034 (72\%) & 0.55 & \multicolumn{1}{c|}{18.98} 
& 0.0106 (72\%) & 0.61 & 18.00 
& 0.0456 (67\%) & 0.70 & \multicolumn{1}{c|}{16.94} 
& 0.1132 (59\%) & 0.94 & 16.52 \\
\cline{2-14}
& \multicolumn{1}{c|}{GMAT+CTPC-CDE}
& 0.0038 (69\%) & 0.51 & \multicolumn{1}{c|}{18.82} 
& 0.0112 (71\%) & 0.66 & 18.18 
& 0.0460 (67\%) & 0.80 & \multicolumn{1}{c|}{17.41} 
& 0.1140 (59\%) & 1.05 & 17.03 \\
\cline{2-14}
& \multicolumn{1}{c|}{GMAT+CTPC-CDE++}
& 0.0038 (69\%) & 0.59 & \multicolumn{1}{c|}{19.11} 
& 0.0122 (68\%) & 0.75 & 18.41 
& 0.0486 (65\%) & 0.85 & \multicolumn{1}{c|}{17.46} 
& 0.1127 (60\%) & 1.09 & 17.01 \\
\midrule
\multirow{6}{*}{\parbox[c][1cm][c]{2.5cm}{\centering \textbf{Average} \\ (All Periods)}}
& \multicolumn{1}{c|}{GMAT}
& 0.0111 (---) & --- & \multicolumn{1}{c|}{---}
& 0.0384 (---) & --- & ---
& 0.1374 (---) & --- & \multicolumn{1}{c|}{---}
& 0.2704 (---) & --- & --- \\
\cline{2-14}
& \multicolumn{1}{c|}{GMAT+Latent $\text{ODE}^{\star}$}
& 0.0040 (64\%) & 75.91 & \multicolumn{1}{c|}{37.54} 
& 0.0141 (63\%) & 248.44 & 37.73
& 0.0687 (49\%) & 3601.10 & \multicolumn{1}{c|}{37.93} 
& 0.1631 (39\%) & 20000.04 & 38.03 \\
\cline{2-14}
& \multicolumn{1}{c|}{GMAT+Latent $\text{ODE}^{\dagger}$}
& 0.0031 (72\%) & 4.10 & \multicolumn{1}{c|}{23.36}
& 0.0149 (61\%) & 31.97 & 22.64
& 0.0730 (46\%) & 567.35 & \multicolumn{1}{c|}{22.48}
& 0.1706 (36\%) & 3408.66 & 22.60 \\
\cline{2-14}
& \multicolumn{1}{c|}{GMAT+CTPC-ODE}
& 0.0030 (72\%) & 0.45 & \multicolumn{1}{c|}{18.27} 
& 0.0108 (71\%) & 0.55 & 17.49
& 0.0434 (68\%) & 0.71 & \multicolumn{1}{c|}{16.79} 
& 0.1073 (60\%) & 1.14 & 16.24 \\
\cline{2-14}
& \multicolumn{1}{c|}{GMAT+CTPC-CDE}
& 0.0037 (66\%) & 0.42 & \multicolumn{1}{c|}{17.88}
& 0.0122 (68\%) & 0.55 & 17.49
& 0.0422 (69\%) & 0.80 & \multicolumn{1}{c|}{16.82}
& 0.0994 (63\%) & 1.11 & 16.51 \\
\cline{2-14}
& \multicolumn{1}{c|}{GMAT+CTPC-CDE++}
& 0.0037 (67\%) & 0.42 & \multicolumn{1}{c|}{17.89}
& 0.0113 (70\%) & 0.58 & 17.39
& 0.0406 (70\%) & 0.82 & \multicolumn{1}{c|}{16.71}
& 0.0967 (64\%) & 1.15 & 16.39 \\
\bottomrule

\end{tabular}
\end{adjustbox}
\label{tab:main_results}
\end{table*}

\subsection{Long-Horizon Forecasting Performance}\label{sec:long_horizon}
Table~\ref{tab:main_results} reports accuracy (MSE), calibration ($\bar d^2$), and sharpness ($-\log|\Sigma|$) across six test windows and four forecast horizons in both interpolation and extrapolation regimes.
The relevant details are given in Appendix~\ref{adx:train_test}.
Across all horizons, our CTPC models consistently improve both accuracy and calibration relative to GMAT and Latent ODE baselines. We further analyze the role of latent-induced variability in Table~\ref{tab:latent_variance}, where we compare calibration using total predictive variance ($\bar{\boldsymbol{\Sigma}}$) against using only mean predictive variance ($\frac{1}{K}\sum_{k=1}^{K}\boldsymbol{\Sigma}_t^{(k)}$).

\paragraph{Accuracy improvements.}
All learned Correctors reduce MSE relative to GMAT across horizons, demonstrating that modeling forecast error dynamics improves deterministic forecasting.
However, the magnitude of improvement depends strongly on architecture: CTPC variants achieve the largest and most consistent reductions, particularly at long horizons where error accumulation dominates.
For example, at the longest horizon (5760M), the averaged MSE across periods decreases from $0.2704$ (GMAT) to $0.1073$ (CTPC-ODE), $0.0994$ (CTPC-CDE), and $0.0967$ (CTPC-CDE++), indicating substantial long-horizon error mitigation.
Moreover, replacing a NODE decoder with an NCDE decoder yields further gains, suggesting that continuous-time control-path conditioning in the decoder improves long-horizon error evolution modeling.

\paragraph{Calibration and Latent ODE baselines.}
A key observation is that Latent ODE baselines can achieve non-trivial MSE reduction while producing severely miscalibrated uncertainty.
In particular, $\text{Latent ODE}^{\star}$ yields extremely large $\bar d^2$ values (often $\gg 1$ and reaching orders of magnitude larger), indicating catastrophic overconfidence despite lower MSE.
This behavior is expected as optimizing MSE encourages point-estimate accuracy and does not constrain predictive covariance to match residual statistics, so the model can collapse uncertainty while still reducing error.
Replacing MSE with CRPS in $\text{Latent ODE}^{\dagger}$ improves calibration substantially (reducing $\bar d^2$ by orders of magnitude), but still remains far from well-calibrated at long horizons.

\paragraph{Calibration and CTPC variants.}
In contrast, CTPC variants maintain $\bar d^2$ values close to unity across horizons (typically $\approx 0.3$--$1.5$ depending on horizon), indicating well-calibrated uncertainty propagation in the open-loop forecast regime.
At the same time, CTPC produces sharp distributions as reflected by consistently large $-\log|\Sigma|$, demonstrating that improved calibration is not obtained by inflating covariance.
Overall, these results validate the proposed training objective (CRPS + NLL + KL) and probabilistic prediction head for long-horizon uncertainty-aware forecasting.

\paragraph{Decoder choice.}
Comparing CTPC-ODE to CTPC-CDE isolates the effect of using an NCDE decoder.
CTPC-CDE improves MSE and calibration across horizons and test windows in the extrapolation regime, suggesting that the NCDE decoder better captures the temporal evolution of forecast errors under continuous-time control paths.
Finally, CTPC-CDE++ provides the most flexible formulation, extending to missing-feature settings (via spline-based control paths) while retaining strong accuracy and calibration, making it the most general variant of our proposed CTPCs in Table~\ref{tab:architectures}.

\section{Conclusion}\label{sec:conclusion}
Physics-based propagators, such as GMAT, provide strong deterministic forecasts but lack reliable uncertainty estimates in open-loop, long-horizon forecast regimes where errors accumulate without corrective observations. We proposed \emph{Continuous-Time Probabilistic Correctors} (CTPC), a Predictor--Corrector framework that learns continuous-time probabilistic error dynamics and wraps around an existing deterministic simulator to improve accuracy while producing full-covariance uncertainty estimates. On real-world spacecraft ephemeris data (NASA CDDIS) across six rolling test windows and horizons up to 4 days, CTPC consistently reduces MSE relative to GMAT and yields markedly better-calibrated uncertainty than Latent ODE baselines. In particular, the standard Latent ODE objective (MSE+KL) can reduce error but becomes severely overconfident, whereas CTPC maintains a normalized Mahalanobis distance closer to the ideal value of one while remaining sharp. We also isolate the source of calibration gains: moment aggregation over $K$ latent samples (total predictive variance) improves calibration compared to using only the mean predictive variance, confirming that latent sampling provides an ensemble-like effect critical for long-horizon uncertainty propagation.

\paragraph{Limitations and future work.}
This work focuses on a single spacecraft forecasting dataset and evaluates calibration via normalized Mahalanobis distance and sharpness via log-determinant; broader validation on downstream risk metrics useful for spacecraft collision avoidance (e.g., conjunction probability sensitivity) is needed. Future directions include (i) extending CTPC to explicitly incorporate parameter uncertainty and/or environmental covariates (space weather), (ii) coupling the CTPCs with differentiable physics for end-to-end adaptation, and (iii) developing formal verification and safety guarantees for the proposed CTPC variants.

\appendix

\section{Moment Aggregation}\label{adx:app_inf}

At inference time, we draw $K$ Monte Carlo samples from the latent posterior,
$s_L^{(k)}\sim \mathcal{N}(\boldsymbol{\mu}_L,\boldsymbol{\Sigma}_L)$, and decode each sample to obtain a probabilistic prediction
$(\boldsymbol{\mu}_t^{(k)},\boldsymbol{\Sigma}_t^{(k)})$ over the forecast error $\hat{\mathbf{e}}_{\text{ECI}}(t)\in\mathbb{R}^3$.
This induces a mixture distribution
\begin{equation}
p(\hat{\mathbf{e}}_{\text{ECI}}(t) \mid \mathbf{z}_{d}(t))
\approx \frac{1}{K}\sum_{k=1}^{K}\mathcal{T}_{\nu}\!\left(\hat{\mathbf{e}}_{\text{ECI}}(t);\boldsymbol{\mu}_t^{(k)},\boldsymbol{\Sigma}_t^{(k)}\right)
\end{equation}
For $\nu > 2$, the Student-$t$ distribution admits finite first and second moments; we therefore summarize the mixture via moment matching, yielding a Gaussian approximation with matched mean and covariance. The aggregated mean is
\begin{equation}
\bar{\boldsymbol{\mu}}_t=\frac{1}{K}\sum_{k=1}^{K}\boldsymbol{\mu}_t^{(k)},
\end{equation}
and the aggregated covariance decomposes as
\begin{equation}
\bar{\boldsymbol{\Sigma}}_t
=
\underbrace{\frac{1}{K}\sum_{k=1}^{K}\boldsymbol{\Sigma}_t^{(k)}}_{\text{within-sample (aleatoric)}}
+
\underbrace{\frac{1}{K}\sum_{k=1}^{K}
\big(\boldsymbol{\mu}_t^{(k)}-\bar{\boldsymbol{\mu}}_t\big)
\big(\boldsymbol{\mu}_t^{(k)}-\bar{\boldsymbol{\mu}}_t\big)^{\!\top}}_{\text{between-sample (latent-induced)}}.
\end{equation}
The second term captures variability induced by latent sampling and is \emph{epistemic-like}, though it does not truly represent parameter uncertainty.

This aggregation mirrors the law of total variance used in deep ensembles \cite{wilson2022bayesiandeeplearningprobabilistic}, which suggests training multiple probabilistic neural networks to quantify uncertainty. However, we obtain an ensemble-like effect by taking $K$ samples from the latent space and propagating them through a single decoder equipped with a single probabilistic prediction head \cite{lakshminarayanan2017simplescalablepredictiveuncertainty}.

\section{CTPC-ODE}\label{adx:ctpc_ode}

Here, we describe the CTPC-ODE variant of our proposed Corrector, shown in Fig.~\ref{fig:CTPC_ODE}.

\paragraph{Encoder (ODE-RNN).}
The encoder processes past forecast errors from the GMAT Predictor to infer a latent-space distribution.
Unlike NCDEs, the hidden state of an ODE-RNN \cite{rubanova2019latentodesirregularlysampledtime} is only continuous \emph{between} observation times. Specifically, the encoder hidden state $\mathbf{z}_e : [t_{i-1}, t_i] \rightarrow \mathbb{R}^{h_e}$ evolves according to the NODE parameterized by an MLP ($f_{\theta_e}$), and is discretely updated at observation times via a GRU. Formally, between two consecutive observations $t_{i-1}$ and $t_i$, the hidden state evolves as

\begin{align}
\mathbf{z}_e(t)
&= \mathbf{z}_e(t_{i-1})
+ \int_{t_{i-1}}^{t} f_{\theta_e}\!\left(\mathbf{z}_e(s)\right)\, ds,
\quad t \in (t_{i-1}, t_i),
\end{align}

yielding a pre-update state $\tilde{\mathbf{z}}_e(t_i)$. At each observation time $t_i$, the hidden state is updated using the observed forecast error $\mathbf{e}_{\text{ECI}}(t_i)$ via
\begin{align}
    \mathbf{z}_e(t_i) &= \mathrm{GRU}_{\phi}\!\left(\tilde{\mathbf{z}}_e(t_i), \mathbf{e}_{\text{ECI}}(t_i)\right).
\end{align}

After sequentially processing the past error trajectory $\mathbf{e}_{\text{ECI}}(t_{0:T^\prime})$, the final hidden state
$\mathbf{z}_e(t_{T'})$ is mapped to the parameters of a latent Gaussian distribution via a linear projection,
\begin{align}
    (\boldsymbol{\mu}_L, \boldsymbol{\Sigma}_L)
    &\leftarrow \mathbf{W}_{h_e \rightarrow L}\,\mathbf{z}_e(t_{T'}),
    \quad \mathbf{W}_{h_e \rightarrow L} \in \mathbb{R}^{2L \times h_e},
\end{align}
where $\boldsymbol{\Sigma}_L$ is parameterized as a diagonal covariance matrix.

\begin{figure*}[ht]
  \centering
  \includegraphics[width=\textwidth]{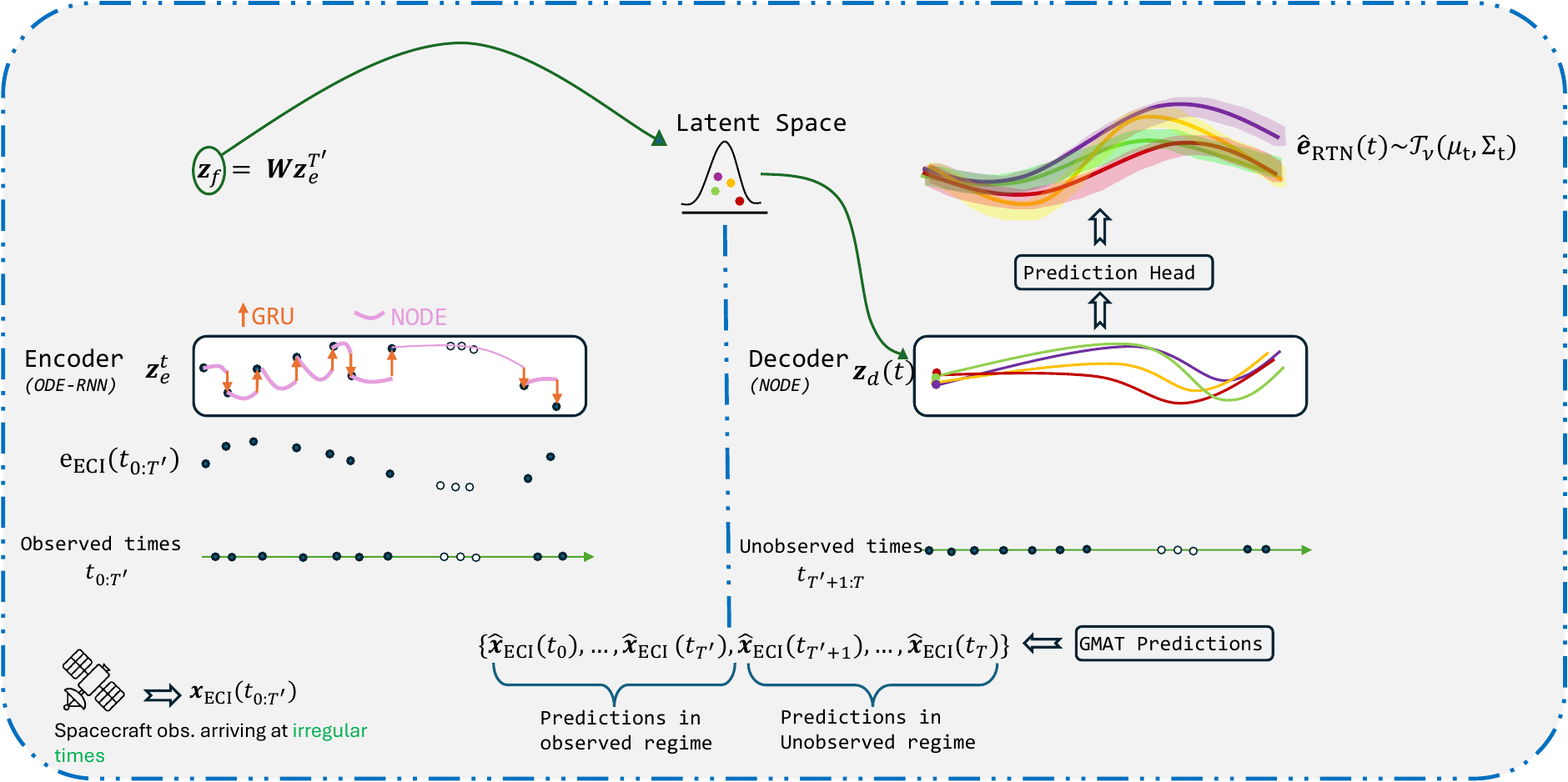}
  \caption{The CTPC-ODE variant of CTPC with ODE-RNN as the encoder and NODE as the decoder. CTPC-ODE can model the error in the RTN frame, $\hat{\mathbf{e}}_{\text{RTN}}(t)$, only.}
  \label{fig:CTPC_ODE}
\end{figure*}

\paragraph{Decoder (NODE).}
The decoder propagates samples from the latent space forward in continuous time to model the evolution of forecast errors over the forecast horizon, $[t_{T^{\prime}+1}, T]$.
Given a latent sample $s_L \sim \mathcal{N}(\boldsymbol{\mu}_L,\boldsymbol{\Sigma}_L)$, the initial decoder hidden state is obtained via a linear mapping
\begin{equation}
\mathbf{z}_d(t_{T'+1}) \leftarrow \mathbf{W}_{L \rightarrow h_d}\, s_L,
\qquad
\mathbf{W}_{L \rightarrow h_d} \in \mathbb{R}^{h_d \times L}.
\end{equation}

The decoder hidden state then evolves (without control path) according to the NODE parameterized by $f_{\theta_d}$,
\begin{equation}
\mathbf{z}_d(t)
=
\mathbf{z}_d(t_{T'+1})
+
\int_{t_{T'+1}}^{t}
f_{\theta_d}\!\left(\mathbf{z}_d(s)\right)\, ds,
\qquad
t \in (t_{T'+1}, t_T].
\end{equation}

\paragraph{Prediction head.} At each forecast time $t\in(t_{T^{\prime}+1}, t_T]$, a probabilistic MLP ($f_{\theta_p}$) maps the decoder hidden state $\mathbf{z}_d(t)$ to the parameters of a multivariate Student-$t$ distribution over the error $\hat{\mathbf{e}}_{\text{ECI}}(t)\in\mathbb{R}^3$. The prediction head outputs a mean $\boldsymbol{\mu}_t\in\mathbb{R}^3$ and a lower-triangular matrix $\mathbf{L}_t\in\mathbb{R}^{3\times 3}$ with strictly positive diagonal entries, which parametrizes a full covariance matrix

\begin{equation}
\boldsymbol{\Sigma}_t = \mathbf{L}_t \mathbf{L}_t^{\top}.
\end{equation}
We then model
\begin{equation}
\hat{\mathbf{e}}_{\text{ECI}}(t) \mid \mathbf{z}_d(t) \sim \mathcal{T}_{\nu}\!\left(\boldsymbol{\mu}_t, \boldsymbol{\Sigma}_t\right)
\end{equation}
where $\nu>0$ denotes the degrees of freedom. Let $\mathbf{r}_t = \mathbf{e}_{\text{ECI}}(t) - \boldsymbol{\mu}_t$ and define $\mathbf{y}_t$ via the triangular solve as

\begin{equation}
\mathbf{y}_t = \mathbf{L}_t^{-1}\mathbf{r}_t,
\qquad
\delta_t = \|\mathbf{y}_t\|_2^2
= \mathbf{r}_t^{\top}\boldsymbol{\Sigma}_t^{-1}\mathbf{r}_t,
\end{equation}
where $\delta_t$ is the squared Mahalanobis distance. The log-determinant is computed from the Cholesky factor as
\begin{equation}
\log|\boldsymbol{\Sigma}_t| = 2\sum_{i=1}^{3}\log\big((\mathbf{L}_t)_{ii}\big).
\end{equation}

\section{CTPC-CDE}\label{adx:ctpc_cde}

Here, we describe the CTPC-ODE variant of our proposed Corrector, shown in Fig.~\ref{fig:CTPC_CDE}.

\paragraph{Encoder (ODE-RNN).}
The encoder processes past forecast errors from the GMAT Predictor to infer a latent-space distribution.
Unlike NCDEs, the hidden state of an ODE-RNN \cite{rubanova2019latentodesirregularlysampledtime} is only continuous \emph{between} observation times. Specifically, the encoder hidden state $\mathbf{z}_e : [t_{i-1}, t_i] \rightarrow \mathbb{R}^{h_e}$ evolves according to a neural ordinary differential equation (NODE) parameterized by $f_{\theta_e}$, and is discretely updated at observation times via a GRU. Formally, between two consecutive observations $t_{i-1}$ and $t_i$, the hidden state evolves as

\begin{align}
\mathbf{z}_e(t)
&= \mathbf{z}_e(t_{i-1})
+ \int_{t_{i-1}}^{t} f_{\theta_e}\!\left(\mathbf{z}_e(s)\right)\, ds,
\quad t \in (t_{i-1}, t_i),
\end{align}

yielding a pre-update state $\tilde{\mathbf{z}}_e(t_i)$. At each observation time $t_i$, the hidden state is updated using the observed forecast error $\mathbf{e}_{\text{ECI}}(t_i)$ via
\begin{align}
    \mathbf{z}_e(t_i) &= \mathrm{GRU}_{\phi}\!\left(\tilde{\mathbf{z}}_e(t_i), \mathbf{e}_{\text{ECI}}(t_i)\right).
\end{align}

After sequentially processing the past error trajectory $\mathbf{e}_{\text{ECI}}(t_{0:T^\prime})$, the final hidden state
$\mathbf{z}_e(t_{T'})$ is mapped to the parameters of a latent Gaussian distribution via a linear projection,
\begin{align}
    (\boldsymbol{\mu}_L, \boldsymbol{\Sigma}_L)
    &\leftarrow \mathbf{W}_{h_e \rightarrow L}\,\mathbf{z}_e(t_{T'}),
    \quad \mathbf{W}_{h_e \rightarrow L} \in \mathbb{R}^{2L \times h_e},
\end{align}
where $\boldsymbol{\Sigma}_L$ is parameterized as a diagonal covariance matrix.

\begin{figure*}[ht]
  \centering
  \includegraphics[width=\textwidth]{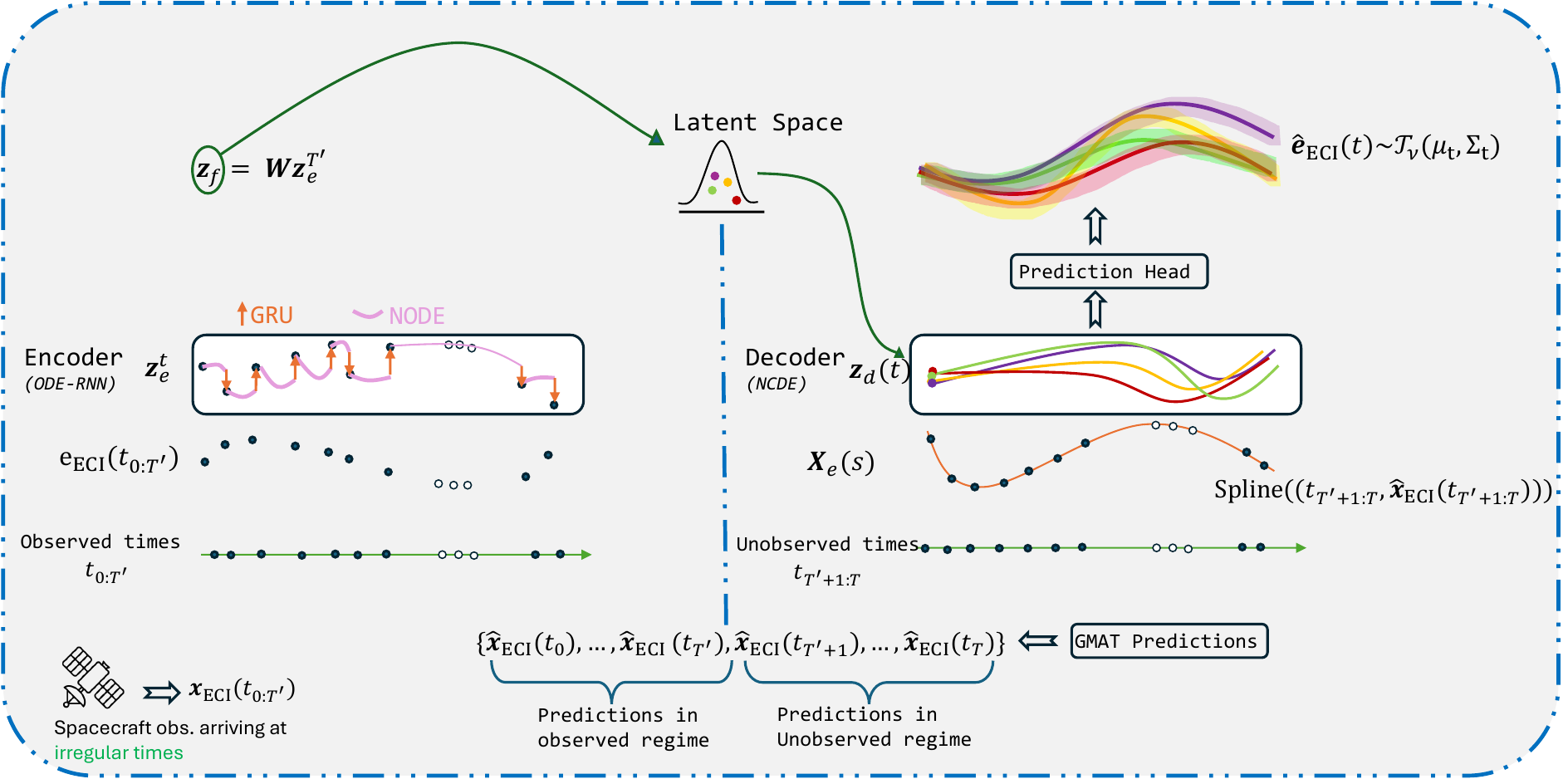}
  \caption{The CTPC-CDE variant of CTPC with ODE-RNN as the encoder and NCDE as the decoder. Though the predicted error $\hat{\mathbf{e}}_{\text{ECI}}(t)$ is shown in the ECI frame, CTPC-CDE can model the error equally well in the RTN frame.}
  \label{fig:CTPC_CDE}
\end{figure*}

\paragraph{Decoder (NCDE).} The decoder, consisting of an NCDE, processes the forecasts of the GMAT Predictor to predict their time-varying probabilistic corrections. Let $\mathbf{z}_{d}: [t_{T^{\prime}+1},t_{T}] \rightarrow \mathbb{R}^{h_d}$ represent the evolution of the hidden states of the decoder. Mathematically, for $t \in (t_{T^{\prime}+1},t_{T}]$,

\begin{align}
    \mathbf{z}_{d}(t) &= \mathbf{z}_{d}(t_{T^{\prime}+1}) + \int_{t_{T^{\prime}+1}}^{t} f_{\theta_{d}}(\mathbf{z}_{d}(s))\,d\mathbf{X}_{d}(s)
    \label{eq:ncde_decoder_1}
\end{align}

The control path, $\mathbf{X}_{d}(t): [t_{T^{\prime}+1},t_{T}] \rightarrow \mathbb{R}^{4}$, is a spline interpolant constructed using the GMAT Predictor's forecasts $\mathbf{\hat{x}}_{\text{ECI}}(t_{T^{\prime}+1:T})$. The initial hidden state $\mathbf{z}_{d}(t_{{T^{\prime}+1}}) \in \mathbb{R}^{h_d}$ is sampled as follows:

\begin{align}
    s_L &\sim \mathcal{N}(\boldsymbol{\mu}_{L}, \boldsymbol{\Sigma}_{L}) \\
    \mathbf{z}_{d}(t_{T^{\prime}+1}) &= \mathbf{W}_{L \rightarrow h_d}s_L \,\,\, \text{with} \,\,\, \mathbf{W}_{L \rightarrow h_d} \in \mathbb{R}^{h_d \times L}
\end{align}

The $\mathbf{W}_{L\rightarrow h_d}$ is a linear mapping from the latent space to the hidden space of the decoder. By utilizing the samples $\mathbf{z}_{d}(t_{T^{\prime}+1})$ as initial hidden states driven by the control path $\mathbf{X}_d(t)$ consisting of deterministic forecasts, the decoder learns to generate diverse evolutions of hidden states $\mathbf{z}_d(t)$ over the forecast horizon $t\in(t_{T^{\prime}+1}, t_T]$, effectively capturing the uncertainty associated with the forecasts.

\paragraph{Prediction head.} At each forecast time $t\in(t_{T^{\prime}+1}, t_T]$, a probabilistic multi-layer perceptron ($f_{\theta_p}$) maps the decoder hidden state $\mathbf{z}_d(t)$ to the parameters of a multivariate Student-$t$ distribution over the error $\hat{\mathbf{e}}_{\text{ECI}}(t)\in\mathbb{R}^3$. The prediction head outputs a mean $\boldsymbol{\mu}_t\in\mathbb{R}^3$ and a lower-triangular matrix $\mathbf{L}_t\in\mathbb{R}^{3\times 3}$ with positive diagonal entries, which parameterizes a full covariance matrix

\begin{equation}
\boldsymbol{\Sigma}_t = \mathbf{L}_t \mathbf{L}_t^{\top}.
\end{equation}
We then model
\begin{equation}
\hat{\mathbf{e}}_{\text{ECI}}(t) \mid \mathbf{z}_d(t) \sim \mathcal{T}_{\nu}\!\left(\boldsymbol{\mu}_t, \boldsymbol{\Sigma}_t\right)
\end{equation}
where $\nu>0$ denotes the degrees of freedom. Let $\mathbf{r}_t = \mathbf{e}_{\text{ECI}}(t) - \boldsymbol{\mu}_t$ and define $\mathbf{y}_t$ via the triangular solve as

\begin{equation}
\mathbf{y}_t = \mathbf{L}_t^{-1}\mathbf{r}_t,
\qquad
\delta_t = \|\mathbf{y}_t\|_2^2
= \mathbf{r}_t^{\top}\boldsymbol{\Sigma}_t^{-1}\mathbf{r}_t,
\end{equation}
where $\delta_t$ is the squared Mahalanobis distance. The log-determinant is computed from the Cholesky factor as
\begin{equation}
\log|\boldsymbol{\Sigma}_t| = 2\sum_{i=1}^{3}\log\big((\mathbf{L}_t)_{ii}\big).
\end{equation}

\section{Additional Analysis and Missing Proofs}\label{adx:missing_theory}
This section presents additional analysis and missing proofs to characterize theoretical guarantees for the proposed framework.
We first show some properties for the nominal dynamics by imposing the following two assumptions.
\begin{assumption}\label{assumption_1}
    The nominal dynamics is Lipschitz continuous, i.e., $\|f(\mathbf{x})-f(\mathbf{y})\|\leq L_f\|\mathbf{x}-\mathbf{y}\|$, $\forall \mathbf{x}, \mathbf{y}$.
\end{assumption}
This is standard in continuous-time control and guarantees existence, uniqueness, and stability bounds for both the true system and the predictor.
\begin{assumption}\label{assumption_2}
    The disturbance process satisfies $\mathbb E\|\eta_t\|^2\leq \sigma^2_\eta<\infty,\;\forall t\in[t_0,t_T]$.
\end{assumption}
Finite second moments are required to analyze mean-square stability and covariance propagation. Assumptions~\ref{assumption_1} and \ref{assumption_2} are primarily for establishing the mean-squared error growth, which motivates the learned corrector in this work. We first characterize the error dynamics in the following result.

\noindent\textbf{Notation and setup.} We consider a continuous-time dynamical system over a finite horizon $[t_0,t_T]$.
We first define the true system state as $\mathbf{x}_t\in\mathbb R^D$, which evolves as
$
    \dot{\mathbf{x}}_\text{ECI}(t)=f(\mathbf{x}_\text{ECI}(t))+\eta_t,
$
where $f:\mathbb R^D\to\mathbb R^D$ is the nominal dynamics, and $\eta_t$ is an unknown disturbance capturing unmodeled physics, parameter mismatch, and numerical errors. As shown before, a deterministic physics-based simulator (e.g., GMAT) produces open-loop forecasts
$    \dot{\hat{\mathbf{x}}}_\text{ECI}(t)=f(\hat{\mathbf{x}}_\text{ECI}(t))$, $\hat{\mathbf{x}}_\text{ECI}(t_0)=\mathbf{x}_\text{ECI}(t_0)$. Recall that $\mathbf{e}_\text{ECI}(t)=\mathbf{x}_\text{ECI}(t)-\hat{\mathbf{x}}_\text{ECI}(t)$.

\begin{lemma}\label{lemma_1}
    The forecast error satisfies
    \begin{equation}
        \dot{\mathbf{e}}_\text{ECI}(t) = f(\mathbf{x}_\text{ECI}(t))-f(\hat{\mathbf{x}}_\text{ECI}(t))+\eta_t.
    \end{equation}
\end{lemma}
\begin{proof}
    By definition, we have the following result
    \begin{equation}
    \begin{split}
        \dot{\mathbf{e}}_\text{ECI}(t)&=\dot{\mathbf{x}}_\text{ECI}(t)-\dot{\hat{\mathbf{x}}}_\text{ECI}(t)\\&=(f(\mathbf{x}_\text{ECI}(t))+\eta_t)-f(\hat{\mathbf{x}}_\text{ECI}(t))\\&=f(\mathbf{x}_\text{ECI}(t))-f(\hat{\mathbf{x}}_\text{ECI}(t))+\eta_t,
    \end{split}
    \end{equation}
    which completes the proof.
\end{proof}
\begin{lemma}\label{lemma_2}
    Under Assumptions~\ref{assumption_1} and~\ref{assumption_2}, we have
    \begin{equation}
        \mathbb E\|\mathbf{e}_\text{ECI}(t)\|^2\leq e^{2L_ft}\mathbb E\|\mathbf e_\text{ECI}(0)\|^2+\frac{\sigma^2_\eta}{L_f}(e^{2L_ft}-1).
    \end{equation}
\end{lemma}
\begin{proof}
    From Lemma~\ref{lemma_1} and Assumption~\ref{assumption_1}, we have
    \begin{equation}
        \|\dot{\mathbf{e}}_\text{ECI}(t)\|\leq L_f\|\mathbf{e}_\text{ECI}(t)\|+\|\eta_t\|.
    \end{equation}
    Squaring and taking expectations yields
    \begin{equation}
        \frac{d}{dt}\mathbb E\|\mathbf{e}_\text{ECI}(t)\|^2\leq 2L_f\mathbb E\|\mathbf{e}_\text{ECI}(t)\|^2 + 2\mathbb E\|\eta_t\|^2.
    \end{equation}
    Using Assumption~\ref{assumption_2},
    \begin{equation}
        \frac{d}{dt}\mathbb E\|\mathbf{e}_\text{ECI}(t)\|^2\leq 2L_f\mathbb E\|\mathbf{e}_\text{ECI}(t)\|^2 + 2\sigma^2_\eta.
    \end{equation}
    Solving this linear differential inequality via Gr\"onwall's inequality yields the stated bound.
\end{proof}
Lemma~\ref{lemma_2} implies that without a corrector, the forecast errors grow exponentially with the horizon.

\textbf{Proof of Lemma~\ref{lemma_3}.} We next prove Lemma~\ref{lemma_3} to show the decoder stability.
\begin{proof}
    From the decoder NCDE:
    \begin{equation}
    \mathbf{z}_d(t)=\mathbf{z}_d(t_{T'+1})+\int_{t_{T'+1}}^tf_{\theta_d}(\mathbf{z}_d(s))d\mathbf{X}_d(s)\;\forall t\in(t_{T'+1},t_T].
    \end{equation}
    Taking norms yields
    \begin{equation}
        \|\mathbf{z}_d(t)\|\leq \|\mathbf{z}_d(t_{T'+1})\| + L_d\int_{t_{T'+1}}^t\|\mathbf{z}_d(s)\|d\|\mathbf{X}_d(s)\|.
    \end{equation}
    We then apply Gr\"onwall's inequality to obtain the following relationship
    \begin{equation}
        \|\mathbf{z}_d(t)\|\leq\|\mathbf{z}_d(t_{T'+1})\|\exp(L_d\|\mathbf{X}_d\|_{TV;[t_{T'},t]}).
    \end{equation}
    Squaring, taking expectations, and applying Assumption~\ref{assumption_7} obtains the desired result.
\end{proof}
We then present a result to characterize the calibration consistency. It enables the proper scoring rule optimality.
\begin{lemma}\label{lemma_4}
    The Student-$t$ negative log-likelihood (NLL) + Continuous Ranked Probability Score (CRPS) + Kullback--Leibler (KL) loss is uniquely minimized by the true conditional distribution.
\end{lemma}
\begin{proof}
    First, the Student-$t$ NLL is strictly proper for distributions with finite variance. Second, CRPS is strictly proper for continuous distributions. Third, KL divergence is a strictly proper scoring rule for continuous probability distributions. Thus, their sum preserves strict properness, and the minimizer equals the true conditional law.
\end{proof}
\textbf{Proof of Theorem~\ref{theorem_1}.} We show the proof for Theorem~\ref{theorem_1} next.
\begin{proof}
    By Lemma~\ref{lemma_4} and realizability, we have
    \begin{equation}
        \bm{\mu}_t=\mathbb E[\mathbf{e}_\text{ECI}(t)|\mathcal{F}_{T'}],\quad\Sigma_t=\text{Cov}[\mathbf{e}_\text{ECI}(t)|\mathcal{F}_{T'}].
    \end{equation}
    Then:
    \begin{equation}
        \begin{split}
            \mathbb E[\textbf{r}^\top\Sigma^{-1}\textbf{r}]&=\mathbb E[\text{tr}(\Sigma^{-1}\textbf{r}\textbf{r}^\top)]\\&=\text{tr}(\Sigma^{-1}\mathbb E[\textbf{r}\textbf{r}^\top])\\&=\text{tr}(I_D)=D.
        \end{split}
    \end{equation}
\end{proof}
\textbf{Proof of Theorem~\ref{theorem_2}.} We show the proof for Theorem~\ref{theorem_2} in the following.
\begin{proof}
    By definition, we have
    \begin{equation}
        \tilde{\mathbf{e}}_\text{ECI}(t)=\mathbf{e}_\text{ECI}(t)-\bm{\mu}_t.
    \end{equation}
    Taking the squared norm and expectation,
    \begin{equation}\label{eq_22}
        \mathbb E\|\tilde{\mathbf{e}}_\text{ECI}(t)\|^2=\mathbb E\|\mathbf{e}_\text{ECI}(t)\|^2 - 2\mathbb E[\mathbf{e}^\top_\text{ECI}(t)\bm{\mu}_t]+\|\bm{\mu}_t\|^2.
    \end{equation}
    Since $\bm{\mu}_t=\mathbb E[\mathbf e_t|\mathcal{F}_{T'}]$,
    \begin{equation}\label{eq_23}
        \mathbb E[\mathbf{e}^\top_\text{ECI}(t)\bm{\mu}_t]=\bm{\mu}_t^\top\bm{\mu}_t.
    \end{equation}
    Substituting Eq.~\ref{eq_23} into Eq.~\ref{eq_22} obtains the desired result.
\end{proof}
\textbf{Proof of Theorem~\ref{theorem_3}.} We next show the proof for Theorem~\ref{theorem_3}.
\begin{proof}
    From Assumption~\ref{assumption_6}, we have
    \begin{equation}
        \|\Sigma_t\|\leq L^2_\Sigma\mathbb E\|\mathbf{z}_d(t)\|^2.
    \end{equation}
    We then substitute Lemma~\ref{lemma_3} into the last inequality, obtaining
    \begin{equation}
        \|\Sigma_t\|\leq L^2_\Sigma M_0\exp(2L_d\|\mathbf{X}_d\|_{TV;[t_{T'},t]}).
    \end{equation}
    If the latent includes additive stochasticity with variance $\sigma^2_{\mathbf{z}_d}$, then:
    \begin{equation}
        \mathbb E\|\mathbf{z}_d(t)\|^2\leq M_0e^{2L_d} + \sigma^2_{\mathbf{z}_d}\int^t_{t_{T'}}e^{2L_d\|\mathbf{X}_d\|_{TV;[s,t]}}ds.
    \end{equation}
    Multiplying by $L^2_\Sigma$ yields the result.
\end{proof}

We now move to the system-level stability induced by input--output stability. We first impose an assumption for the Lipschitz continuity of the encoder.
\begin{assumption}\label{assumption_3}
    The NCDE encoder vector field $f_{\theta_e}$ satisfies $\|f_{\theta_e}(\mathbf{z}_e)-f_{\theta_e}(\mathbf{z}'_e)\|\leq L_e\|\mathbf{z}_e-\mathbf{z}'_e\|, \forall \mathbf{z}_e, \mathbf{z}'_e$, and $L_e>0$.
\end{assumption}
The following assumption is presented to bound the latent distribution stability.
\begin{assumption}\label{assumption_8}
    The encoder induces a latent initialization mapping such that $\mathbf{z}_e(t_{T'})\mapsto\mathbf{z}_d(t_{T'+1})$ satisfying $\mathbb E\|\mathbf{z}_d(t_{T'+1})-\mathbf{z}'_d(t_{T'+1})\|^2\leq L_z^2\|\mathbf{z}_e(t_{T'})-\mathbf{z}'_e(t_{T'})\|^2$, where $L_z>0$.
\end{assumption}
We first provide the encoder stability as follows.
\begin{lemma}\label{lemma_5}
    Let $\mathbf{z}_e(t)$ and $\mathbf{z}'_e(t)$ be encoder states driven by control paths $\mathbf{X}_e$ and $\mathbf{X}'_e$. Then:
    \begin{equation}
        \|\mathbf{z}_e(t_{T'})-\mathbf{z}'_e(t_{T'})\|\leq L_ee^{L_e\|\mathbf{X}_e\|_{TV}}\|\mathbf{X}_e-\mathbf{X}'_e\|_{TV}.
    \end{equation}
\end{lemma}
\begin{proof}
    Recall the encoder dynamics as
\begin{equation}
    \mathbf{z}_e(t)=\mathbf{z}_e(t_0)+\int_{t_0}^tf_{\theta_e}(\mathbf{z}_e(s))d\mathbf{X}_e(s)\;\forall t\in(t_0,t_{T'}].
\end{equation}
Similarly, we have
\begin{equation}
    \mathbf{z}'_e(t)=\mathbf{z}_e(t_0)+\int_{t_0}^tf_{\theta_e}(\mathbf{z}'_e(s))d\mathbf{X}'_e(s)\;\forall t\in(t_0,t_{T'}].
\end{equation}
We subtract $\mathbf{z}'_e(t)$ from $\mathbf{z}_e(t)$ to obtain
\begin{equation}
\begin{split}
    \mathbf{z}_e(t)-\mathbf{z}'_e(t)&=\int_{t_0}^t(f_{\theta_e}(\mathbf{z}_e(s))-f_{\theta_e}(\mathbf{z}'_e(s)))d\mathbf{X}_e(s)\\& + \int_{t_0}^tf_{\theta_e}(\mathbf{z}'_e(s))d(\mathbf{X}_e-\mathbf{X}'_e)(s).
\end{split}
\end{equation}
Taking norms, we have
\begin{equation}
\begin{split}
    \|\mathbf{z}_e(t)-\mathbf{z}'_e(t)\| &\leq L_e\int_{t_0}^t\|\mathbf{z}_e(s)-\mathbf{z}'_e(s)\|\,|d\mathbf{X}_e(s)| \\
    &\quad +L_e\|\mathbf{X}_e-\mathbf{X}'_e\|_{TV}.
\end{split}
\end{equation}
Applying Gr\"onwall's inequality for control paths yields the desired result.
\end{proof}
Likewise, we present the decoder propagation of latent perturbations as follows.
\begin{lemma}\label{lemma_6}
    Let $\mathbf{z}_d(t)$ and $\mathbf{z}'_d(t)$ be decoder trajectories with initializations $\mathbf{z}_d(t_{T'+1})$ and $\mathbf{z}'_d(t_{T'+1})$. Then:
    \begin{equation}
    \begin{split}
        \|\mathbf{z}_d(t)-\mathbf{z}'_d(t)\| &\leq \|\mathbf{z}_d(t_{T'+1})-\mathbf{z}'_d(t_{T'+1})\| \\
        &\quad\times \exp(L_d\|\mathbf{X}_d\|_{TV;[t_{T'},t]}).
    \end{split}
    \end{equation}
\end{lemma}
\begin{proof}
    Recall the decoder dynamics as follows:
\begin{equation}
    \mathbf{z}_d(t)=\mathbf{z}_d(t_{T'+1})+\int_{t_{T'+1}}^tf_{\theta_d}(\mathbf{z}_d(s))d\mathbf{X}_d(s)\;\forall t\in(t_{T'+1},t_T].
\end{equation}
Similarly, we have
\begin{equation}
    \mathbf{z}'_d(t)=\mathbf{z}'_d(t_{T'+1})+\int_{t_{T'+1}}^tf_{\theta_d}(\mathbf{z}'_d(s))d\mathbf{X}_d(s)\;\forall t\in(t_{T'+1},t_T].
\end{equation}
Subtracting yields
\begin{equation}
\begin{split}
    \mathbf{z}_d(t)-\mathbf{z}'_d(t)&=\mathbf{z}_d(t_{T'+1})-\mathbf{z}'_d(t_{T'+1})\\&+\int_{t_{T'+1}}^t(f_{\theta_d}(\mathbf{z}_d(s))-f_{\theta_d}(\mathbf{z}'_d(s)))d\mathbf{X}_d(s).
\end{split}
\end{equation}
Taking norms, we obtain:
\begin{equation}
\begin{split}
    \|\mathbf{z}_d(t)-\mathbf{z}'_d(t)\|&\leq \|\mathbf{z}_d(t_{T'+1})-\mathbf{z}'_d(t_{T'+1})\| \\&+ L_d\int_{t_{T'+1}}^t\|\mathbf{z}_d(s)-\mathbf{z}'_d(s)\||d\mathbf{X}_d(s)|.
\end{split}
\end{equation}
Applying Gr\"onwall's inequality completes the proof.
\end{proof}
Given the above two lemmas, we now establish the end-to-end Predictor--Corrector stability.
\begin{theorem}\label{theorem_4}
    Let $\mathbf{X}_e$ and $\mathbf{X}'_e$ be two past-observation control paths. Let $(\bm{\mu}_t,\Sigma_t)$ and $(\bm{\mu}'_t,\Sigma'_t)$ be the corresponding corrector outputs. Then for all $t\in[t_{T'},T]$, we have the following relationships:
    \begin{equation}
        \begin{split}
            &\|\bm{\mu}_t-\bm{\mu}'_t\|\leq C_\mu\|\mathbf{X}_e-\mathbf{X}'_e\|_{TV}\\
            &\|\Sigma_t-\Sigma'_t\|\leq C_\Sigma\|\mathbf{X}_e-\mathbf{X}'_e\|_{TV},
        \end{split}
    \end{equation}
    where $C_\mu=L_\mu L_z L_e e^{L_e\|\mathbf{X}_e\|_{TV}}e^{L_d\|\mathbf{X}_d\|_{TV}}$ and analogously for $C_\Sigma$.
\end{theorem}
\begin{proof}
    By Lemma~\ref{lemma_5}, we have
    \begin{equation}
        \|\mathbf{z}_e(t_{T'})-\mathbf{z}'_e(t_{T'})\|\leq L_ee^{L_e\|\mathbf{X}_e\|_{TV}}\|\mathbf{X}_e-\mathbf{X}'_e\|_{TV}.
    \end{equation}
    By Assumption~\ref{assumption_8}, we can also get
    \begin{equation}
        \mathbb E\|\mathbf{z}_d(t_{T'+1})-\mathbf{z}'_d(t_{T'+1})\|\leq L_z\|\mathbf{z}_e(t_{T'})-\mathbf{z}'_e(t_{T'})\|.
    \end{equation}
    By Lemma~\ref{lemma_6}, we have the decoder propagation as follows
    \begin{equation}
    \begin{split}
        \|\mathbf{z}_d(t)-\mathbf{z}'_d(t)\| &\leq \|\mathbf{z}_d(t_{T'+1})-\mathbf{z}'_d(t_{T'+1})\| \\
        &\quad\times \exp(L_d\|\mathbf{X}_d\|_{TV;[t_{T'},t]}).
    \end{split}
    \end{equation}
    Combining the last three inequalities yields
    \begin{equation}\label{eq_42}
        \|\mathbf{z}_d(t)-\mathbf{z}'_d(t)\|\leq L_z L_ee^{L_e\|\mathbf{X}_e\|_{TV}}e^{L_d\|\mathbf{X}_d\|_{TV;[t_{T'},t]}}\|\mathbf{X}_e-\mathbf{X}'_e\|_{TV}.
    \end{equation}
    Then by Assumption~\ref{assumption_6}, we can obtain that
    \begin{equation}\label{eq_43}
        \|\bm{\mu}(\mathbf{z}_d)-\bm{\mu}(\mathbf{z}'_d)\|\leq L_\mu\|\mathbf{z}_d-\mathbf{z}'_d\|.
    \end{equation}
    and
    \begin{equation}\label{eq_44}
        \|\Sigma(\mathbf{z}_d)-\Sigma(\mathbf{z}'_d)\|\leq L_\Sigma\|\mathbf{z}_d-\mathbf{z}'_d\|.
    \end{equation}
    Substituting Eq.~\ref{eq_42} into Eqs.~\ref{eq_43} and \ref{eq_44} completes the proof.
\end{proof}
System-level stability characterizes the robustness of the full encoder--decoder pipeline to perturbations in past observations. The derived input--output stability result shows that small changes in the warm-up data lead to bounded changes in the predicted mean and covariance over the entire forecast horizon. This guarantees that uncertainty estimates and corrections are not fragile with respect to observation noise or minor trajectory variations. Unlike long-horizon covariance growth, which is governed solely by the decoder dynamics, system-level stability concerns the consistency of latent initialization and its propagation through the model. While not required for forecast-time stability, this result ensures end-to-end well-posedness and supports the reliability of calibration across trajectories, complementing the decoder-level guarantees.

\section{Coordinate Transformations (ECI \& RTN)}\label{adx:coord}

\subsection{Coordinate-Frame Robustness}\label{adx:coord_robust}

In spacecraft dynamics, forecast errors may be represented in different coordinate frames, most commonly the Earth-Centered Inertial (ECI) frame and the rotating Radial--Transverse--Normal (RTN) frame. A desirable property of a probabilistic Corrector is robustness to such coordinate transformations, i.e., the ability to model error dynamics consistently regardless of the chosen frame.

We observe a fundamental distinction between CTPC variants that utilize NODE- and NCDE-based decoders. The NODE-based variant (CTPC-ODE) models forecast errors as the solution of an ordinary differential equation,
\(
\dot{\mathbf{z}}_d(t) = f_{\theta_d} (\mathbf{z}_d(t)),
\)
where the hidden dynamics are not controlled by an explicit control path. Therefore, the learned dynamics are implicitly tied to the coordinate system in which the errors are expressed. In practice, this restricts CTPC-ODE to modeling errors in the RTN frame, where the error dynamics exhibit approximately stationary structure aligned with orbital geometry.

In contrast, NCDE-based variants (CTPC-CDE and CTPC-CDE++) evolve hidden states according to
\(
\mathrm{d}\mathbf{z}_d(t) = f_{\theta_d}\big(\mathbf{z}_d(t)\big)\,\mathrm{d}\mathbf{X}_d(t),
\)
where the control path $\mathbf{X}_d(t)$ is a spline interpolant consisting of the GMAT Predictor's deterministic forecasts. Because coordinate transformations act directly on the control path rather than on uncontrolled hidden dynamics, NCDEs naturally adapt to different coordinate representations. Consequently, CTPC-CDE and CTPC-CDE++ can model forecast errors consistently in both ECI and RTN frames without modifying the underlying architecture.

\subsection{Error Transformation Between ECI and RTN Frames}\label{adx:frame_conversion}

Let $\mathbf{x}_{\text{ECI}}(t)\in\mathbb{R}^3$ denote the observed spacecraft position in the ECI frame and
$\hat{\mathbf{x}}_{\text{ECI}}(t)\in\mathbb{R}^3$ the corresponding GMAT-predicted position.\footnote{
GMAT internally propagates a full Cartesian state; here we denote by $\hat{\mathbf{x}}_{\text{ECI}}(t)$ the position component of that state.
}
The forecast error in ECI is
\begin{equation}
\mathbf{e}_{\mathrm{ECI}}(t) = \mathbf{x}_{\text{ECI}}(t) - \hat{\mathbf{x}}_{\text{ECI}}(t).
\end{equation}

Let $\hat{\mathbf{r}}(t) := \hat{\mathbf{x}}_{\text{ECI}}(t)$ and
$\hat{\mathbf{v}}(t) := \dot{\hat{\mathbf{x}}}_{\text{ECI}}(t)$
denote the GMAT Predictor's forecast position and velocity in the ECI frame at time $t$.
We construct the RTN orthonormal triad from these forecast states. The radial unit vector is
\begin{equation}
\hat{\mathbf{r}}_{\mathrm{RTN}}(t) = \frac{\hat{\mathbf{r}}(t)}{\|\hat{\mathbf{r}}(t)\|}.
\end{equation}
To obtain the along-track (transverse) direction, we remove the radial component of the predicted velocity,
\begin{equation}
\mathbf{t}(t) = \hat{\mathbf{v}}(t) - \big(\hat{\mathbf{v}}(t)^\top \hat{\mathbf{r}}_{\mathrm{RTN}}(t)\big)\hat{\mathbf{r}}_{\mathrm{RTN}}(t),
\qquad
\hat{\mathbf{t}}_{\mathrm{RTN}}(t) = \frac{\mathbf{t}(t)}{\|\mathbf{t}(t)\|},
\end{equation}
and define the orbit-normal unit vector as
\begin{equation}
\hat{\mathbf{n}}_{\mathrm{RTN}}(t) =
\frac{\hat{\mathbf{r}}_{\mathrm{RTN}}(t)\times \hat{\mathbf{t}}_{\mathrm{RTN}}(t)}
{\|\hat{\mathbf{r}}_{\mathrm{RTN}}(t)\times \hat{\mathbf{t}}_{\mathrm{RTN}}(t)\|}.
\end{equation}
The rotation matrix mapping ECI vectors to RTN coordinates is then
\begin{equation}
\mathbf{R}_{\mathrm{ECI}\rightarrow \mathrm{RTN}}(t)
=
\begin{bmatrix}
\hat{\mathbf{r}}_{\mathrm{RTN}}(t)^\top\\
\hat{\mathbf{t}}_{\mathrm{RTN}}(t)^\top\\
\hat{\mathbf{n}}_{\mathrm{RTN}}(t)^\top
\end{bmatrix},
\end{equation}
so that for any vector $\mathbf{u}(t)\in\mathbb{R}^3$ expressed in ECI, its RTN representation is $\mathbf{u}_{\mathrm{RTN}}(t)=\mathbf{R}_{\mathrm{ECI}\rightarrow\mathrm{RTN}}(t)\mathbf{u}_{\mathrm{ECI}}(t)$. (We use the convention that the RTN basis vectors form the rows of $\mathbf{R}_{\mathrm{ECI}\rightarrow \mathrm{RTN}}$, hence $\mathbf{u}_{\mathrm{RTN}}=\mathbf{R}_{\mathrm{ECI}\rightarrow\mathrm{RTN}}(t)\mathbf{u}_{\mathrm{ECI}}$.)

Since the rotation matrix is time-varying and depends on the underlying observed trajectory in the ECI frame, error dynamics in ECI are generally non-stationary, whereas RTN-aligned errors often exhibit more structured behavior. CTPC variants with an NCDE-based decoder explicitly condition their hidden dynamics $\mathbf{z}_d(t)$ on the GMAT Predictor's forecast through the control path $\mathbf{X}_d(t)$, enabling them to learn error evolution in either frame without architectural changes.

Overall, this analysis highlights that coordinate-frame robustness is an architectural property of CTPC variants containing NCDE-based decoders, rather than a consequence of loss design.

\section{Effect of Latent-Induced Variability on Calibration}\label{adx:latent_var}

\begin{table*}[!tb]
\centering
\small 
\caption{We present the results of two variants of our proposed CTPC to demonstrate the efficacy of latent-induced variability on the calibration. For comparison, we report results from Table~\ref{tab:main_results} with total variance ($\bar{\Sigma}$) against results with mean predictive variance ($\frac{1}{K}\sum_{k=1}^{K}\boldsymbol{\Sigma}_t^{(k)}$). $\ddagger$ = total variance ($\bar{\Sigma}$). The results are reported on errors in the RTN frame.}
\vspace{0.15cm}
\renewcommand{\arraystretch}{1.2}
\setlength{\tabcolsep}{4pt} 
\begin{adjustbox}{max width=\textwidth} 
\begin{tabular}{c|c|cccccc|cccccc}
\toprule
\multicolumn{1}{c|}{\multirow{3}{*}{\parbox[c][1cm][c]{2.5cm}{\centering \textbf{Test}\\ \textbf{Period}}}} & 
\multicolumn{1}{c|}{\multirow{3}{*}{\textbf{Model}}} &
\multicolumn{6}{c|}{\textbf{Interpolation Horizons}} &
\multicolumn{6}{c}{\textbf{Extrapolation Horizons}} \\
\cline{3-14} 
& & \multicolumn{3}{c|}{1000M (0D 16H 40M)} & \multicolumn{3}{c|}{2000M (1D 9H 20M)} & \multicolumn{3}{c|}{4000M (2D 18H 40M)} & \multicolumn{3}{c}{5760M (4D 0H 0M)} \\
\cline{3-14} 
& & MSE (\%$\downarrow$) & $\bar{d}^2$ & \multicolumn{1}{c|}{$-\log|\Sigma|$} & MSE (\%$\downarrow$) & $\bar{d}^2$ & $-\log|\Sigma|$ & MSE (\%$\downarrow$) & $\bar{d}^2$ & \multicolumn{1}{c|}{$-\log|\Sigma|$} & MSE (\%$\downarrow$) & $\bar{d}^2$ & $-\log|\Sigma|$ \\
\midrule
\multirow{3}{*}{\parbox[c][1cm][c]{2.5cm}{\centering 2017-02-15\\ \textbf{to} 2017-02-28}}
& \multicolumn{1}{c|}{GMAT}
& 0.0105 (---) & --- & \multicolumn{1}{c|}{---} 
& 0.0396 (---) & --- & --- 
& 0.1465 (---) & --- & \multicolumn{1}{c|}{---} 
& 0.2963 (---) & --- & --- \\
\cline{2-14}
& \multicolumn{1}{c|}{GMAT+CTPC-ODE}
& 0.0021 (80\%) & 0.52 & \multicolumn{1}{c|}{18.43} 
& 0.0081 (79\%) & 0.71 & 17.85 
& 0.0414 (71\%) & 0.96 & \multicolumn{1}{c|}{17.12} 
& 0.1174 (60\%) & 1.34 & 16.78 \\
\cline{2-14}
& \multicolumn{1}{c|}{GMAT+$\text{CTPC-ODE}^{\ddagger}$}
& 0.0021 (80\%) & 0.45 & \multicolumn{1}{c|}{17.85} 
& 0.0081 (79\%) & 0.59 & 17.19 
& 0.0414 (71\%) & 0.75 & \multicolumn{1}{c|}{16.40} 
& 0.1174 (60\%) & 1.03 & 16.04 \\
\cline{2-14}
& \multicolumn{1}{c|}{GMAT+CTPC-CDE}
& 0.0033 (69\%) & 0.51 & \multicolumn{1}{c|}{18.59} 
& 0.0097 (75\%) & 0.75 & 18.12 
& 0.0330 (77\%) & 1.14 & \multicolumn{1}{c|}{17.51} 
& 0.0881 (70\%) & 1.70 & 17.23 \\
\cline{2-14}
& \multicolumn{1}{c|}{GMAT+$\text{CTPC-CDE}^{\ddagger}$}
& 0.0033 (69\%) & 0.43 & \multicolumn{1}{c|}{17.91} 
& 0.0097 (75\%) & 0.60 & 17.34 
& 0.0330 (77\%) & 0.84 & \multicolumn{1}{c|}{16.64} 
& 0.0881 (70\%) & 1.21 & 16.33 \\
\midrule
\multirow{3}{*}{\parbox[c][1cm][c]{2.5cm}{\centering 2017-04-15\\ \textbf{to} 2017-04-30}}
& \multicolumn{1}{c|}{GMAT}
& 0.0121 (---) & --- & \multicolumn{1}{c|}{---} 
& 0.0393 (---) & --- & --- 
& 0.1437 (---) & --- & \multicolumn{1}{c|}{---} 
& 0.2841 (---) & --- & --- \\
\cline{2-14}
& \multicolumn{1}{c|}{GMAT+CTPC-ODE}
& 0.0038 (68\%) & 0.72 & \multicolumn{1}{c|}{19.09} 
& 0.0139 (64\%) & 0.81 & 18.23 
& 0.0562 (60\%) & 0.99 & \multicolumn{1}{c|}{17.23} 
& 0.1312 (53\%) & 1.34 & 16.81 \\
\cline{2-14}
& \multicolumn{1}{c|}{GMAT+$\text{CTPC-ODE}^{\ddagger}$}
& 0.0038 (68\%) & 0.51 & \multicolumn{1}{c|}{18.25} 
& 0.0139 (64\%) & 0.58 & 17.34 
& 0.0562 (60\%) & 0.74 & \multicolumn{1}{c|}{16.37} 
& 0.1312 (53\%) & 1.02 & 15.97 \\
\cline{2-14}
& \multicolumn{1}{c|}{GMAT+CTPC-CDE}
& 0.0050 (58\%) & 0.52 & \multicolumn{1}{c|}{17.83} 
& 0.0155 (60\%) & 0.71 & 17.49 
& 0.0544 (62\%) & 1.15 & \multicolumn{1}{c|}{16.93} 
& 0.1261 (55\%) & 1.70 & 16.67 \\
\cline{2-14}
& \multicolumn{1}{c|}{GMAT+$\text{CTPS-CDE}^{\ddagger}$}
& 0.0050 (58\%) & 0.42 & \multicolumn{1}{c|}{17.32} 
& 0.0155 (60\%) & 0.55 & 16.86 
& 0.0544 (62\%) & 0.87 & \multicolumn{1}{c|}{16.22} 
& 0.1261 (55\%) & 1.27 & 15.94 \\
\midrule
\multirow{3}{*}{\parbox[c][1cm][c]{2.5cm}{\centering 2017-06-15\\ \textbf{to} 2017-06-30}}
& \multicolumn{1}{c|}{GMAT}
& 0.0126 (---) & --- & \multicolumn{1}{c|}{---} 
& 0.0468 (---) & --- & --- 
& 0.1647 (---) & --- & \multicolumn{1}{c|}{---} 
& 0.3200 (---) & --- & --- \\
\cline{2-14}
& \multicolumn{1}{c|}{GMAT+CTPC-ODE}
& 0.0030 (76\%) & 0.64 & \multicolumn{1}{c|}{19.32} 
& 0.0097 (79\%) & 0.86 & 18.49 
& 0.0336 (79\%) & 1.09 & \multicolumn{1}{c|}{17.47} 
& 0.0859 (73\%) & 1.44 & 17.01 \\
\cline{2-14}
& \multicolumn{1}{c|}{GMAT+$\text{CTPC-ODE}^{\ddagger}$}
& 0.0030 (76\%) & 0.48 & \multicolumn{1}{c|}{18.51} 
& 0.0097 (79\%) & 0.65 & 17.62 
& 0.0336 (79\%) & 0.84 & \multicolumn{1}{c|}{16.62} 
& 0.0859 (73\%) & 1.12 & 16.18 \\
\cline{2-14}
& \multicolumn{1}{c|}{GMAT+CTPC-CDE}
& 0.0034 (72\%) & 0.50 & \multicolumn{1}{c|}{18.33} 
& 0.0114 (75\%) & 0.77 & 17.94 
& 0.0359 (78\%) & 1.33 & \multicolumn{1}{c|}{17.56} 
& 0.0830 (74\%) & 2.02 & 17.39 \\
\cline{2-14}
& \multicolumn{1}{c|}{GMAT+$\text{CTPC-CDE}^{\ddagger}$}
& 0.0034 (72\%) & 0.39 & \multicolumn{1}{c|}{17.71} 
& 0.0114 (75\%) & 0.56 & 17.15 
& 0.0359 (78\%) & 0.99 & \multicolumn{1}{c|}{16.67} 
& 0.0830 (74\%) & 1.32 & 16.46 \\
\midrule
\multirow{3}{*}{\parbox[c][1cm][c]{2.5cm}{\centering 2017-08-15\\ \textbf{to} 2017-08-31}}
& \multicolumn{1}{c|}{GMAT}
& 0.0100 (---) & --- & \multicolumn{1}{c|}{---} 
& 0.0324 (---) & --- & --- 
& 0.1075 (---) & --- & \multicolumn{1}{c|}{---} 
& 0.2003 (---) & --- & --- \\
\cline{2-14}
& \multicolumn{1}{c|}{GMAT+CTPC-ODE}
& 0.0037 (63\%) & 0.45 & \multicolumn{1}{c|}{18.46} 
& 0.0136 (57\%) & 0.57 & 18.01
& 0.0413 (61\%) & 0.76 & \multicolumn{1}{c|}{17.37} 
& 0.0885 (55\%) & 1.02 & 17.07 \\
\cline{2-14}
& \multicolumn{1}{c|}{GMAT+$\text{CTPC-ODE}^{\ddagger}$}
& 0.0037 (63\%) & 0.37 & \multicolumn{1}{c|}{17.85} 
& 0.0136 (57\%) & 0.44 & 17.28
& 0.0413 (61\%) & 0.54 & \multicolumn{1}{c|}{16.54} 
& 0.0885 (55\%) & 0.69 & 16.20 \\
\cline{2-14}
& \multicolumn{1}{c|}{GMAT+CTPC-CDE}
& 0.0041 (58\%) & 0.44 & \multicolumn{1}{c|}{18.46} 
& 0.0150 (53\%) & 0.57 & 18.08 
& 0.0430 (60\%) & 0.80 & \multicolumn{1}{c|}{17.51} 
& 0.0883 (55\%) & 1.07 & 17.21 \\
\cline{2-14}
& \multicolumn{1}{c|}{GMAT+$\text{CTPC-CDE}^{\ddagger}$}
& 0.0041 (58\%) & 0.38 & \multicolumn{1}{c|}{17.93} 
& 0.0150 (53\%) & 0.45 & 17.41 
& 0.0430 (60\%) & 0.57 & \multicolumn{1}{c|}{16.72} 
& 0.0883 (55\%) & 0.74 & 16.40 \\
\midrule
\multirow{3}{*}{\parbox[c][1cm][c]{2.5cm}{\centering 2017-10-15\\ \textbf{to} 2017-10-30}}
& \multicolumn{1}{c|}{GMAT}
& 0.0088 (---) & --- & \multicolumn{1}{c|}{---} 
& 0.0334 (---) & --- & --- 
& 0.1201 (---) & --- & \multicolumn{1}{c|}{---} 
& 0.2400 (---) & --- & --- \\
\cline{2-14}
& \multicolumn{1}{c|}{GMAT+CTPC-ODE}
& 0.0025 (71\%) & 0.46 & \multicolumn{1}{c|}{18.84} 
& 0.0094 (71\%) & 0.62 & 18.31 
& 0.0423 (64\%) & 0.95 & \multicolumn{1}{c|}{17.70} 
& 0.1078 (55\%) & 1.42 & 17.41 \\
\cline{2-14}
& \multicolumn{1}{c|}{GMAT+$\text{CTPC-ODE}^{\ddagger}$}
& 0.0025 (71\%) & 0.36 & \multicolumn{1}{c|}{18.15} 
& 0.0094 (71\%) & 0.47 & 17.53 
& 0.0423 (64\%) & 0.70 & \multicolumn{1}{c|}{16.86} 
& 0.1078 (55\%) & 1.02 & 16.55 \\
\cline{2-14}
& \multicolumn{1}{c|}{GMAT+CTPC-CDE}
& 0.0028 (68\%) & 0.49 & \multicolumn{1}{c|}{19.36} 
& 0.0105 (68\%) & 0.73 & 18.88 
& 0.0412 (65\%) & 1.08 & \multicolumn{1}{c|}{18.24} 
& 0.0970 (59\%) & 1.56 & 17.90 \\
\cline{2-14}
& \multicolumn{1}{c|}{GMAT+$\text{CTPC-CDE}^{\ddagger}$}
& 0.0028 (68\%) & 0.37 & \multicolumn{1}{c|}{18.61} 
& 0.0105 (68\%) & 0.52 & 18.01 
& 0.0412 (65\%) & 0.74 & \multicolumn{1}{c|}{17.28} 
& 0.0970 (59\%) & 1.05 & 16.92 \\
\midrule
\multirow{3}{*}{\parbox[c][1cm][c]{2.5cm}{\centering 2017-12-15\\ \textbf{to} 2017-12-31}}
& \multicolumn{1}{c|}{GMAT}
& 0.0124 (---) & --- & \multicolumn{1}{c|}{---} 
& 0.0390 (---) & --- & --- 
& 0.1420 (---) & --- & \multicolumn{1}{c|}{---} 
& 0.2819 (---) & --- & --- \\
\cline{2-14}
& \multicolumn{1}{c|}{GMAT+CTPC-ODE}
& 0.0034 (72\%) & 0.72 & \multicolumn{1}{c|}{19.74} 
& 0.0106 (72\%) & 0.83 & 18.86 
& 0.0456 (67\%) & 0.94 & \multicolumn{1}{c|}{17.80} 
& 0.1132 (59\%) & 1.24 & 17.36 \\
\cline{2-14}
& \multicolumn{1}{c|}{GMAT+$\text{CTPC-ODE}^{\ddagger}$}
& 0.0034 (72\%) & 0.55 & \multicolumn{1}{c|}{18.98} 
& 0.0106 (72\%) & 0.61 & 18.00 
& 0.0456 (67\%) & 0.70 & \multicolumn{1}{c|}{16.94} 
& 0.1132 (59\%) & 0.94 & 16.52 \\
\cline{2-14}
& \multicolumn{1}{c|}{GMAT+CTPC-CDE}
& 0.0038 (69\%) & 0.65 & \multicolumn{1}{c|}{19.49} 
& 0.0112 (71\%) & 0.85 & 18.95 
& 0.0460 (67\%) & 1.08 & \multicolumn{1}{c|}{18.22} 
& 0.1140 (59\%) & 1.45 & 17.85 \\
\cline{2-14}
& \multicolumn{1}{c|}{GMAT+$\text{CTPC-CDE}^{\ddagger}$}
& 0.0038 (69\%) & 0.51 & \multicolumn{1}{c|}{18.82} 
& 0.0112 (71\%) & 0.66 & 18.18 
& 0.0460 (67\%) & 0.80 & \multicolumn{1}{c|}{17.41} 
& 0.1140 (59\%) & 1.05 & 17.03 \\
\bottomrule

\end{tabular}
\end{adjustbox}
\label{tab:latent_variance}
\end{table*}

Table~\ref{tab:latent_variance} isolates the contribution of latent-induced variability by comparing two variants of CTPC:
(i) using only the predictive mean variance (i.e., the average of per-sample predictive covariances), and
(ii) using the total predictive variance $\hat{\Sigma}$, obtained by aggregating both within-sample (aleatoric) and between-sample (latent-induced) variability.
Across all test periods and horizons, incorporating latent-induced variability consistently improves calibration, as evidenced by normalized squared Mahalanobis distances $\bar{d}^2$ moving closer to the ideal value of one.
In contrast, relying solely on predictive mean variance systematically underestimates uncertainty, leading to $\bar{d}^2 > 1$ and overconfident behavior, particularly at longer horizons.
These results demonstrate that latent sampling produces an ensemble-like effect, critical for achieving well-calibrated uncertainty estimates in long-horizon, open-loop forecasting.

\section{Results in the ECI Frame}\label{adx:eci_frame}

Table~\ref{tab:main_results_eci} evaluates the robustness of two types of CTPC variants, NODE-based and NCDE-based decoders, to coordinate transformations by modeling forecast errors directly in the ECI frame. Specifically, we compare the performance of CTPC-ODE (NODE-based decoder) and CTPC-CDE (NCDE-based decoder). Across all six test windows and horizons, CTPC-ODE fails to provide meaningful improvements over the GMAT baseline. Its average MSE reductions are negligible or negative, and its normalized Mahalanobis distances grow faster with horizon than those of CTPC-CDE, indicating severe miscalibration and extreme overconfidence. In contrast, CTPC-CDE consistently achieves substantial reductions in MSE while maintaining significantly lower (though still $>1$) normalized Mahalanobis distances and sharper uncertainty estimates. The averaged results across all periods show that CTPC-CDE reduces long-horizon MSE by up to 24\% at 5760 minutes while yielding markedly improved calibration relative to CTPC-ODE. These results demonstrate that the NCDE-based decoder enables CTPC to model forecast errors robustly under coordinate transformations, whereas NODE-based decoders do not generalize reliably beyond the RTN frame.

\begin{table*}[!tb]
\centering
\small 
\caption{This table compares CTPC-ODE and CTPC-CDE modeling forecast errors in the ECI frame, using NODE and NCDE decoders, respectively. The objective is to assess the robustness of the NCDE-based decoder variants of CTPC to coordinate transformations. The results are reported on errors in the ECI frame.}
\vspace{0.15cm}
\renewcommand{\arraystretch}{1.2}
\setlength{\tabcolsep}{4pt} 
\begin{adjustbox}{max width=\textwidth} 
\begin{tabular}{c|c|cccccc|cccccc}
\toprule
\multicolumn{1}{c|}{\multirow{3}{*}{\parbox[c][1cm][c]{2.5cm}{\centering \textbf{Test}\\ \textbf{Period}}}} & 
\multicolumn{1}{c|}{\multirow{3}{*}{\textbf{Model}}} &
\multicolumn{6}{c|}{\textbf{Interpolation Horizons}} &
\multicolumn{6}{c}{\textbf{Extrapolation Horizons}} \\
\cline{3-14} 
& & \multicolumn{3}{c|}{1000M (0D 16H 40M)} & \multicolumn{3}{c|}{2000M (1D 9H 20M)} & \multicolumn{3}{c|}{4000M (2D 18H 40M)} & \multicolumn{3}{c}{5760M (4D 0H 0M)} \\
\cline{3-14} 
& & MSE (\%$\downarrow$) & $\bar{d}^2$ & \multicolumn{1}{c|}{$-\log|\Sigma|$} & MSE (\%$\downarrow$) & $\bar{d}^2$ & $-\log|\Sigma|$ & MSE (\%$\downarrow$) & $\bar{d}^2$ & \multicolumn{1}{c|}{$-\log|\Sigma|$} & MSE (\%$\downarrow$) & $\bar{d}^2$ & $-\log|\Sigma|$ \\
\midrule
\multirow{3}{*}{\parbox[c][1cm][c]{2.5cm}{\centering 2017-02-15\\ \textbf{to} 2017-02-28}}
& \multicolumn{1}{c|}{GMAT}
& 0.0106 (---) & --- & \multicolumn{1}{c|}{---} 
& 0.0396 (---) & --- & --- 
& 0.1465 (---) & --- & \multicolumn{1}{c|}{---} 
& 0.2963 (---) & --- & --- \\
\cline{2-14}
& \multicolumn{1}{c|}{GMAT+CTPC-ODE}
& 0.0105 (0.58\%) & 4.41 & \multicolumn{1}{c|}{15.26} 
& 0.0396 (-0.02\%) & 7.88 & 13.28 
& 0.1466 (-0.07\%) & 19.05 & \multicolumn{1}{c|}{11.70} 
& 0.2966 (-0.11\%) & 31.65 & 11.21 \\
\cline{2-14}
& \multicolumn{1}{c|}{GMAT+CTPC-CDE}
& 0.0033 (68\%) & 9.44 & \multicolumn{1}{c|}{16.88} 
& 0.0145 (63\%) & 28.88 & 15.75
& 0.0852 (41\%) & 78.22 & \multicolumn{1}{c|}{14.23} 
& 0.2146 (27\%) & 123.28 & 13.49 \\
\midrule
\multirow{3}{*}{\parbox[c][1cm][c]{2.5cm}{\centering 2017-04-15\\ \textbf{to} 2017-04-30}}
& \multicolumn{1}{c|}{GMAT}
& 0.0121 (---) & --- & \multicolumn{1}{c|}{---} 
& 0.0393 (---) & --- & --- 
& 0.1437 (---) & --- & \multicolumn{1}{c|}{---} 
& 0.2841 (---) & --- & --- \\
\cline{2-14}
& \multicolumn{1}{c|}{GMAT+CTPC-ODE}
& 0.0122 (-0.49\%) & 24.41 & \multicolumn{1}{c|}{15.99} 
& 0.0394 (-0.17\%) & 78.99 & 14.63 
& 0.1438 (-0.05\%) & 268.49 & \multicolumn{1}{c|}{13.66} 
& 0.2842 (-0.02\%) & 455.22 & 13.35 \\
\cline{2-14}
& \multicolumn{1}{c|}{GMAT+CTPC-CDE}
& 0.0042 (65\%) & 19.07 & \multicolumn{1}{c|}{17.81} 
& 0.0158 (59\%) & 48.78 & 16.53 
& 0.0749 (44\%) & 117.56 & \multicolumn{1}{c|}{14.91} 
& 0.2033 (28\%) & 173.57 & 14.17 \\
\midrule
\multirow{3}{*}{\parbox[c][1cm][c]{2.5cm}{\centering 2017-06-15\\ \textbf{to} 2017-06-30}}
& \multicolumn{1}{c|}{GMAT}
& 0.0126 (---) & --- & \multicolumn{1}{c|}{---} 
& 0.0468 (---) & --- & --- 
& 0.1647 (---) & --- & \multicolumn{1}{c|}{---} 
& 0.3200 (---) & --- & --- \\
\cline{2-14}
& \multicolumn{1}{c|}{GMAT+CTPC-ODE}
& 0.0126 (0.15\%) & 22.70 & \multicolumn{1}{c|}{15.54} 
& 0.0470 (-0.33\%) & 96.46 & 14.22 
& 0.1654 (-0.38\%) & 329.27 & \multicolumn{1}{c|}{13.17} 
& 0.3213 (-0.37\%) & 589.23 & 12.83 \\
\cline{2-14}
& \multicolumn{1}{c|}{GMAT+CTPC-CDE}
& 0.0046 (63\%) & 24.32 & \multicolumn{1}{c|}{17.04} 
& 0.0115 (66\%) & 90.73 & 16.07 
& 0.0819 (50\%) & 271.95 & \multicolumn{1}{c|}{14.67} 
& 0.2169 (32\%) & 397.67 & 13.84 \\
\midrule
\multirow{3}{*}{\parbox[c][1cm][c]{2.5cm}{\centering 2017-08-15\\ \textbf{to} 2017-08-31}}
& \multicolumn{1}{c|}{GMAT}
& 0.0100 (---) & --- & \multicolumn{1}{c|}{---} 
& 0.0324 (---) & --- & --- 
& 0.1075 (---) & --- & \multicolumn{1}{c|}{---} 
& 0.2003 (---) & --- & --- \\
\cline{2-14}
& \multicolumn{1}{c|}{GMAT+CTPC-ODE}
& 0.0100 (-0.22\%) & 13.90 & \multicolumn{1}{c|}{15.96} 
& 0.0325 (-0.32\%) & 56.39 & 14.70
& 0.1084 (-0.80\%) & 226.83 & \multicolumn{1}{c|}{13.65} 
& 0.2024 (-1.07\%) & 417.85 & 13.28 \\
\cline{2-14}
& \multicolumn{1}{c|}{GMAT+CTPC-CDE}
& 0.0041 (58.85\%) & 14.81 & \multicolumn{1}{c|}{17.74} 
& 0.0161 (50.19\%) & 43.16 & 16.58 
& 0.0676 (37.11\%) & 91.20 & \multicolumn{1}{c|}{14.88} 
& 0.1648 (17.73\%) & 115.07 & 14.01 \\
\midrule
\multirow{3}{*}{\parbox[c][1cm][c]{2.5cm}{\centering 2017-10-15\\ \textbf{to} 2017-10-30}}
& \multicolumn{1}{c|}{GMAT}
& 0.0088 (---) & --- & \multicolumn{1}{c|}{---} 
& 0.0334 (---) & --- & --- 
& 0.1201 (---) & --- & \multicolumn{1}{c|}{---} 
& 0.2400 (---) & --- & --- \\
\cline{2-14}
& \multicolumn{1}{c|}{GMAT+CTPC-ODE}
& 0.0088 (0.13\%) & 8.56 & \multicolumn{1}{c|}{15.78} 
& 0.0335 (-0.36\%) & 35.48 & 14.62 
& 0.1206 (-0.38\%) & 125.27 & \multicolumn{1}{c|}{13.71} 
& 0.2410 (-0.42\%) & 231.29 & 13.42 \\
\cline{2-14}
& \multicolumn{1}{c|}{GMAT+CTPC-CDE}
& 0.0027 (69\%) & 9.47 & \multicolumn{1}{c|}{17.41} 
& 0.0103 (69\%) & 39.98 & 16.53 
& 0.0743 (38\%) & 163.61 & \multicolumn{1}{c|}{15.27} 
& 0.2115 (11\%) & 309.16 & 14.69 \\
\midrule
\multirow{3}{*}{\parbox[c][1cm][c]{2.5cm}{\centering 2017-12-15\\ \textbf{to} 2017-12-31}}
& \multicolumn{1}{c|}{GMAT}
& 0.0124 (---) & --- & \multicolumn{1}{c|}{---} 
& 0.0390 (---) & --- & --- 
& 0.1420 (---) & --- & \multicolumn{1}{c|}{---} 
& 0.2819 (---) & --- & --- \\
\cline{2-14}
& \multicolumn{1}{c|}{GMAT+CTPC-ODE}
& 0.0124 (0.04\%) & 6.78 & \multicolumn{1}{c|}{16.33} 
& 0.0391 (-0.30\%) & 22.89 & 15.25 
& 0.1430 (-0.68\%) & 78.51 & \multicolumn{1}{c|}{14.32} 
& 0.2843 (-0.83\%) & 133.43 & 14.01 \\
\cline{2-14}
& \multicolumn{1}{c|}{GMAT+CTPC-CDE}
& 0.0045 (63\%) & 8.56 & \multicolumn{1}{c|}{18.01} 
& 0.0154 (60\%) & 27.64 & 17.09 
& 0.0766 (46\%) & 99.23 & \multicolumn{1}{c|}{15.76} 
& 0.1946 (30\%) & 167.44 & 15.06 \\
\midrule
\multirow{3}{*}{\parbox[c][1cm][c]{2.5cm}{\centering \textbf{Average}\\ (All Periods)}}
& \multicolumn{1}{c|}{GMAT}
& 0.0111 (---) & --- & \multicolumn{1}{c|}{---} 
& 0.0384 (---) & --- & --- 
& 0.1374 (---) & --- & \multicolumn{1}{c|}{---} 
& 0.2704 (---) & --- & --- \\
\cline{2-14}
& \multicolumn{1}{c|}{GMAT+CTPC-ODE}
& 0.0111 (0.03\%) & 13.46 & \multicolumn{1}{c|}{15.81} 
& 0.0385 (-0.25\%) & 49.68 & 14.45
& 0.1380 (-0.39\%) & 174.57 & \multicolumn{1}{c|}{13.37} 
& 0.2716 (-0.47\%) & 309.78 & 13.02 \\
\cline{2-14}
& \multicolumn{1}{c|}{GMAT+CTPC-CDE}
& 0.0039 (64\%) & 14.28 & \multicolumn{1}{c|}{17.48 } 
& 0.0139 (61\%) & 46.53 & 16.43 
& 0.0768 (42\%) & 136.96 & \multicolumn{1}{c|}{14.95} 
& 0.2010 (24\%) & 214.20 & 14.21 \\
\bottomrule

\end{tabular}
\end{adjustbox}
\label{tab:main_results_eci}
\end{table*}

\section{Training and Testing}\label{adx:train_test}

We construct training and testing datasets using a rolling-window approach based on overlapping long-horizon trajectories extracted from continuous spacecraft ephemeris data collected by NASA \cite{noll2000doris}.

\paragraph{Training data.}
For each window, the Corrector is trained using trajectories spanning approximately one and a half months.
Specifically, for the first window, we use data from January 1, 2017 to February 15, 2017.
Training trajectories are generated by initializing a trajectory every 15 minutes, starting at 00:00 on January 1, 2017.
Each trajectory is propagated forward for 2500 minutes ($\approx$1.7 days) using the GMAT Predictor, yielding densely sampled overlapping training trajectories.
This procedure yields a large collection of partially overlapping trajectories that capture the medium-horizon forecast-error dynamics during the training period. We refer to the 2500-minute horizon as the \textit{\textbf{interpolation horizon}}, because the Correctors are trained on trajectories of this length, while longer horizons are evaluated in extrapolation.

\paragraph{Testing data.}
The trained Corrector is evaluated on the immediately following 15-day test period (e.g., February 15--28, 2017), which is strictly disjoint in time from the training data.
Test trajectories are initialized every 6 hours, starting at 00:00 on the first test day, and each trajectory is propagated forward for 5760 minutes (4 days) using the GMAT Predictor.
This produces a set of long-horizon test trajectories designed to assess both error accumulation and uncertainty propagation over extended forecast horizons. The 5760-minute horizon constitutes the \textit{\textbf{extrapolation horizon}}, as it extends well beyond the training horizon (2500 minutes) and evaluates the Corrector's ability to propagate uncertainty under long-term open-loop forecasting.

\paragraph{Rolling-window evaluation.}
The above training and testing procedure constitutes one train--test window.
We repeat this process by advancing the window in time (e.g., March/April, May/June), constructing a total of 6 non-overlapping test windows across 2017.
This evaluation protocol deliberately trains the Corrector on shorter-horizon (2500-minute), densely sampled trajectories while testing on substantially longer-horizon (5760-minute), sparsely initialized trajectories, reflecting realistic operational conditions for spacecraft trajectory forecasting.

\section{Implementation}\label{adx:impl}
The algorithms were implemented using $\mathtt{Jax}$ \cite{jax2018github} and $\mathtt{Diffrax}$ \cite{kidger2022neuraldifferentialequations}. The implementation code, trained checkpoints, and a representative Window-1 subset of the dataset are openly available on Zenodo \cite{shahid_2026_20741146}.



\printcredits

\section*{Declaration of competing interest}

The authors declare that they have no known competing financial interests or personal relationships that could have appeared to influence the work reported in this paper.

\section*{Data availability}

A representative subset of the spacecraft trajectory dataset (Window 1), together with the implementation code and trained model checkpoints, is openly available on Zenodo at \url{https://doi.org/10.5281/zenodo.20741146} \cite{shahid_2026_20741146}. The full dataset is derived from NASA's Crustal Dynamics Data Information System (CDDIS), which is publicly available.

\section*{Declaration of generative AI and AI-assisted technologies in the manuscript preparation process}

During the preparation of this work the author(s) used generative AI tools in order to improve the writing clarity, grammar, and formatting of the manuscript. After using this tool/service, the author(s) reviewed and edited the content as needed and take(s) full responsibility for the content of the published article. All technical ideas, illustration designs, model designs, theoretical developments, experiments, and conclusions are the authors' own. No generative AI tools were used to design illustrations, generate experimental results, or analyze data.

\bibliographystyle{model1-num-names}
\bibliography{my_references}

\end{document}